\PassOptionsToPackage{dvipsnames}{xcolor}
\documentclass{article}

\usepackage{microtype}
\usepackage{graphicx}
\usepackage{booktabs} 

\usepackage{hyperref}

\usepackage[accepted]{icml2022}
\usepackage[multiple]{footmisc}
\usepackage{bm}
\usepackage{soul,color}
\usepackage{amsmath}
\usepackage{graphicx}
\usepackage{subcaption}
\usepackage{placeins}
\usepackage{comment}
\newsavebox{\imagebox}
\usepackage{bbm}
\usepackage{refcount}
\usepackage{wrapfig}
\usepackage{multirow}

\usepackage{enumitem}

\usepackage{amsmath}
\usepackage{amssymb}
\usepackage{mathtools}
\usepackage{amsthm}
\usepackage{thmtools, thm-restate}
\usepackage{enumitem}

\usepackage[capitalize,noabbrev]{cleveref}

\theoremstyle{plain}
\newtheorem{theorem}{Theorem}[section]

\newtheorem{lemma}{Lemma}[section]
\newtheorem{corollary}{Corollary}[section]
\theoremstyle{definition}
\newtheorem{definition}{Definition}[section]
\newtheorem{assumption}{Assumption}[section]

\newtheorem{result}{Result}[section]
\theoremstyle{remark}

\newcommand{\HV}{\textsc{HV}}
\newcommand{\HVI}{\textsc{HVI}}
\newcommand{\EHVI}{\textsc{EHVI}}

\newcommand{\qNParego}{$q$\textsc{NParEGO}}
\newcommand{\NEHVI}{\textsc{NEHVI}}
\newcommand{\qNEHVI}{$q$\textsc{NEHVI}}
\newcommand{\TSHVI}{$q\textsc{NEHVI-1}$}
\newcommand{\mvar}{\textsc{MVaR}}
\newcommand{\var}{\textsc{VaR}}
\newcommand{\cvar}{\textsc{CVaR}}

\newcommand{\iep}{\textsc{IEP}}
\newcommand{\mars}{\textsc{MARS}}
\newcommand{\marsnei}{\textsc{MARS-NEI}}
\newcommand{\marsts}{\textsc{MARS-TS}}
\newcommand{\marsucb}{\textsc{MARS-UCB}}
\newcommand{\nei}{\textsc{NEI}}
\newcommand{\weakmvar}{\textsc{Weak-}\mvar{}}

\newcommand{\gpertop}{\diamond}
\newcommand{\gpert}[2]{{#1} \gpertop {#2}}
\newcommand{\xxi}{\gpert{\bm x}{\bm \xi}} 
\newcommand{\xpxi}{\gpert{\bm x'}{\bm \xi}}

\newcommand{\mvarnehvi}{\mvar\textsc{-NEHVI}}
\newcommand{\mvarnehvirff}{\mvar\textsc{-NEHVI}-RFF}
\newcommand{\expnehvirff}{\textsc{Exp-NEHVI-RFF}}

\DeclareMathOperator*{\argmax}{arg\,max}

\usepackage{scalerel,stackengine}
\stackMath
\newcommand\reallywidehat[1]{%
\savestack{\tmpbox}{\stretchto{%
  \scaleto{%
    \scalerel*[\widthof{\ensuremath{#1}}]{\kern-.6pt\bigwedge\kern-.6pt}%
    {\rule[-\textheight/2]{1ex}{\textheight}}
  }{\textheight}%
}{0.5ex}}%
\stackon[1pt]{#1}{\tmpbox}%
}

\newenvironment{enumerate_compact}{
\begin{enumerate}
  \setlength{\itemsep}{4pt}
  \setlength{\parskip}{0pt}
  \setlength{\parsep}{0pt}
}{\end{enumerate}}

\newcommand{\papertitle}{Robust Multi-Objective Bayesian Optimization Under Input Noise}

\icmltitlerunning{\papertitle}

\begin{document}

\twocolumn[
\icmltitle{\papertitle}



\icmlsetsymbol{equal}{*}

\begin{icmlauthorlist}
\icmlauthor{Samuel Daulton}{equal,meta,ox}
\icmlauthor{Sait Cakmak}{equal,meta,ga}
\icmlauthor{Maximilian Balandat}{meta}
\icmlauthor{Michael A. Osborne}{ox}
\icmlauthor{Enlu Zhou}{ga}
\icmlauthor{Eytan Bakshy}{meta}
\end{icmlauthorlist}

\icmlaffiliation{meta}{Meta}
\icmlaffiliation{ox}{University of Oxford}
\icmlaffiliation{ga}{Georgia Institute of Technology}

\icmlcorrespondingauthor{Samuel Daulton}{sdaulton@fb.com}
\icmlcorrespondingauthor{Sait Cakmak}{saitcakmak@fb.com}

\icmlkeywords{Machine Learning, ICML}

\vskip 0.3in
]



\printAffiliationsAndNotice{\icmlEqualContribution} 

\begin{abstract}
Bayesian optimization (BO) is a sample-efficient approach for tuning design parameters to optimize expensive-to-evaluate, black-box performance metrics.
In many manufacturing processes, the design parameters are subject to 
random input noise, resulting in 
a product
that is often less performant than expected.
Although BO methods have been proposed for optimizing a single objective under input noise, no existing method addresses the practical scenario where there are multiple
objectives that are sensitive to input perturbations.
In this work, we propose the first multi-objective BO method that is robust to input noise. 
We formalize 
our goal as optimizing the multivariate value-at-risk (\mvar{}), a risk measure of the uncertain objectives. 
Since directly optimizing \mvar{} is computationally infeasible in many settings, we propose a scalable, theoretically-grounded approach for optimizing \mvar{} using random scalarizations. 
Empirically, we find that our approach significantly outperforms alternative methods and efficiently identifies optimal robust designs that will satisfy specifications across multiple metrics with high probability.
\end{abstract}
\vspace{-5ex}
\section{Introduction}
Scientists and engineers frequently face optimization problems where the goal is to tune a set of parameters to optimize multiple competing black-box objective functions. Typically, no single design is best with respect to every objective. Hence, the goal in multi-objective (MO) optimization is to identify the Pareto frontier (PF) of optimal trade-offs between the objectives and the corresponding optimal designs. These problems are ubiquitous in a variety of domains including manufacturing~\citep{liao2008}, materials design~\citep{ASHBY2000359}, robotics~\citep{Calandra2014ParetoFM}, and machine learning~\citep{izquierdo2021bag,eriksson2021latencyaware}. Obtaining measurements of the objectives often requires resource-intensive simulation or experimentation. Therefore, any practical routine for optimizing such functions must be highly sample-efficient; that is, the method must identify optimal design parameters while only querying the objective functions at a small number of designs. Bayesian optimization (BO)~\citep{shahriari16review} is a popular technique for addressing this class of problems. 
\begin{figure*}[t]
    \centering
    \includegraphics[width=\textwidth]{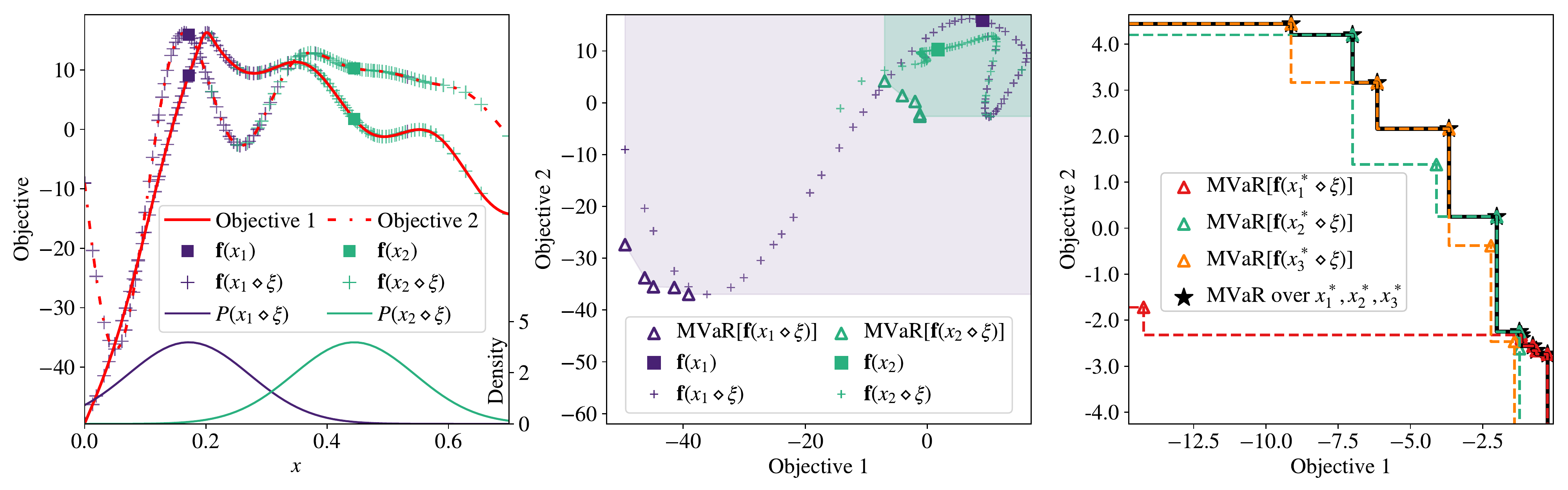}
    \vspace{-4ex}
    \caption{\label{fig:1d_toy}A toy problem where the goal is to tune a single parameter $x$ to maximize two objectives. Left: The nominal values for a non-robust (purple, $x_1$) and a robust design design (green, $x_2$) are indicated using squares. The plus markers illustrate objective values of each design under zero-mean Gaussian input noise with a standard deviation of 0.1. The non-robust design can lead to low objective values under input perturbations. Center: An illustration of the \mvar{} sets of the non-robust and robust designs. The triangles represent a discrete approximation of the \mvar{} set of each design under the input noise distribution. For each point $\bm z$ in the \mvar{} set of a given design, the objective values for that design subject to input perturbations are $\geq \bm z$ with probability $\alpha \geq 0.9$. In other words, the objectives under input noise for each design will fall in the respective shaded region with probability $\geq \alpha$. Under input noise, the non-robust design $x_1$ may result in poor objective values despite yielding better values (relative to the robust design $x_2$) without perturbations. After accounting for input noise, $x_2$ is a more robust solution than $x_1$. Right: The \mvar{} set (black stars) across three optimal designs $x_1^*, x_2^*, x_3^*$ is the set of optimal points across the union of the \mvar{} sets (colored triangles) of each design.}
    \vspace{-2ex}
\end{figure*}

In many real-world applications, the final implementation of the selected design parameters are subject to 
input noise, resulting from uncontrollable perturbations of the input parameters~\citep{BEYER2007}.
For example, many vaccine production processes involve freeze-drying procedures to stabilize active pharmaceutical ingredients to increase storage lifetime~\citep{MORTIER201671, pharma_robust}. Tuning the operating parameters, such as shelf temperature and chamber pressure, can significantly improve the efficiency of the drying step with little reduction in product quality. However, shelf temperature and chamber pressure are subject to uncontrollable random input noise around the \emph{nominal}\footnote{\label{footnote:nominal}We call the design as specified by the decision maker the \emph{nominal design}, and the corresponding value without any perturbations as the \emph{nominal values} for the \emph{nominal objectives}. } input parameters.
Robustness with respect to this input noise is critical. Higher temperatures lead to greater efficiency, but also make the process more sensitive to perturbations in the temperature because if temperature exceeds the critical collapse threshold, there is irreversible product damage. A more conservative temperature that is robust to input noise may be a better choice than a higher temperature that is worse with respect to the \emph{nominal objectives}.
In such high-stakes and high-throughput production pipelines, decision-makers seek to identify robust design parameters that ensure the manufactured products will have high objective quality with high probability. 


Optimization without consideration of input noise can lead to solutions that are catastrophic when subjected to input noise at the implementation stage \citep{DOLTSINIS2004}.
The toy problem that is illustrated in left plot of Figure~\ref{fig:1d_toy} demonstrates a scenario where a non-robust design (depicted as a purple square) is better than the robust design (marked by a green square) with respect to the nominal values. However, the performance metrics for the non-robust design can be significantly worse when the design is subject to input noise. In contrast, the objectives are far less sensitive to input perturbations around the robust design. 

To identify designs that are robust to a noisy performance metric due to input noise or observation noise, the value-at-risk (\var{}) is often used as the robust objective because it provides high-probability performance guarantees \citep{Basel2012}. The $\alpha$ \var{} is the $(1-\alpha)$-quantile (where $0 \leq \alpha \leq 1$) of the cumulative distribution function (CDF) of the performance metric. Under input noise, variation in the objective is induced via the uncertain inputs. Intuitively, the \var{} is the largest value such that the objective value of a given design subject to input perturbations will be greater than that value with probability at least  $\alpha$. Thus, in the context of manufacturing, a design with $\alpha$ \var{} exceeding the target specification produces a yield of at least $\alpha$.

Many recent works have considered BO methods that are robust to input noise in the single objective setting, e.g., by optimizing \var{}~\citep{nguyen2021valueatrisk} or other risk measures~\citep{nes,Cakmak2020borisk, nguyen2021optimizing}. However, no previous work has considered sample-efficient, generally-applicable multi-objective Bayesian optimization (MOBO) methods that are robust to input perturbations. To our knowledge, the only existing MO methods that are robust to input noise are evolutionary algorithms (EA) that often require tens of thousands of evaluations (e.g., \citet{deb2005robustEA}).

In the MO setting, high-probability performance guarantees of a single design can be assessed using the multivariate value-at-risk (\mvar{}) \citep{Prekopa2012MVaR}.
\mvar{} maps a probability value $\alpha$ to a set of vectors where each element provides a lower bound on the the objectives' possible values under input noise with probability $\alpha$.
As illustrated in the center plot of Figure~\ref{fig:1d_toy}, the \mvar{} set of a robust design ($x_2$) often provides significantly better probabilistic lower bounds than the \mvar{} set of a non-robust design ($x_1$).

Similar to the PF in the standard MO setting---where the PF is the set of optimal trade-offs between objectives with no input noise---the \emph{global} \mvar{} set is the set of optimal trade-offs that can be achieved with probability $\geq \alpha$ under input noise across all possible designs. The \mvar{} across a set of 3 designs is illustrated in the right plot of Figure~\ref{fig:1d_toy}.

While \mvar{} provides a natural measure of robust performance guarantees, it is relatively expensive to compute, 
making it computationally challenging to directly optimize \mvar{} in the context of MOBO (see Appendix \ref{appdx:mvar-nehvi} for a discussion). 
In this work, we propose a family of novel, theoretically-grounded methods for optimizing \mvar{} via random scalarizations that mitigates many challenges associated with optimizing \mvar{} directly.

\textbf{Contributions}
\vspace{-3ex}
\begin{enumerate_compact}
    \item We introduce robust MO optimization under input noise, formalize the problem in terms of optimizing global \mvar{}---a novel, probabilistic form of a robust PF---and discuss computational challenges unique to this setting.
    %
    %
    
    \item We derive a novel theoretical connection between the \mvar{} and the \var{} of a particular scalarization of the objectives, which motivates a family of computationally efficient BO methods for identifying \mvar{}-optimal designs using an \mvar{} Approximation based on Random Scalarizations (\mars{}).
    
    \item We demonstrate that \mars{} vastly outperforms non-robust alternatives on a variety of synthetic and real-world robust MO optimization problems, including a pharmaceutical manufacturing application.
    \item We derive and evaluate extensions of our methods to handle 
    expensive-to-evaluate black-box constraints and parallel candidate generation.
\end{enumerate_compact}


\vspace{-2.5ex}
\section{Background} 
\textbf{Multi-Objective Optimization}
In MO optimization, the goal is to, without loss of generality, maximize a vector-valued black-box function $\max_{\bm x \in \mathcal X} \bm f(\bm x)$ where $\bm f(\bm x):= [f_1(\bm x),\ldots, f_M(\bm x)], M\geq 2$, and $\mathcal X \subset \mathbb R^d$ is a compact search space. Often, there may be additional black-box constraints $\bm c(\bm x) \geq \bm 0$, where $\bm c(\bm x)\in \mathbb R^V, ~V > 0$, that must be satisfied. We consider the setting where $\bm f$ and $\bm c$ have no known analytic expressions and no known or observed gradients. The goal is to identify the Pareto frontier (PF) of optimal trade-offs and corresponding Pareto set of optimal designs $\mathcal X^*$.

\textbf{Notation} For vectors $\bm y, \bm y' \in R^M$. Let $\geq$ and $>$ denote the component-wise extensions of the order notations $\geq$ and $>$, i.e., $\bm y \geq \bm y' \iff y_i \geq y_i'~\forall~ i$ and $\bm y> \bm y' \iff y_i >  y_i' ~\forall~ i$, where $\cdot_i$ denotes the $i^\text{th}$ element.

\begin{definition} \label{def:pareto-dominance}
A vector $\bm f(\bm x)$ \emph{Pareto dominates} $\bm f(\bm x')$, denoted by $\bm f(\bm x) \succ \bm f(\bm x')$, if $\bm f(\bm x) \geq \bm f(\bm x')$ and $\exists~ j \in \{1, \dotsc, M\}$ such that $ f^{(j)}(\bm x) > f^{(j)}(\bm x')$. 
\end{definition}

\begin{definition}
The \emph{Pareto frontier} over a set of objective vectors $\mathcal F = \{\bm f(\bm x) ~|~ \bm x \in X \subseteq \mathcal X\}$ is $\textsc{Pareto}(\mathcal F ) = \{\bm f(\bm x) \in \mathcal F~ : ~ \nexists ~\bm x' \in X ~s.t.~  \bm f(\bm x') \succ \bm f(\bm x)\}.$
If there are constraints $\bm c(\bm x)$, elements of \textsc{Pareto} are subject to the additional membership assumption that $\bm c(\bm x) \geq \bm 0$. We call the corresponding set of optimal designs the \emph{Pareto set}.
\end{definition}

Although the true PF $\mathcal P^* = \textsc{Pareto}(\{\bm f(\bm x)\}_{\bm x \in \mathcal X})$ is typically unknown, an MO optimization algorithm can be employed to identify a finite approximation. Hypervolume is a commonly used metric to measure the quality of a PF.

\begin{definition}
The \emph{hypervolume} (\HV{}) \emph{indicator} of a set of points $\mathcal Y \subset \mathbb R^M$ is the $M$-dimensional Lebesgue measure $\lambda_M$ of the region dominated by $\mathcal P := \textsc{Pareto}(\mathcal Y)$ and bounded from below by a reference point $\bm r \in \mathbb R^M$, which we write as $\HV{}(\mathcal Y, \bm r)$.
\end{definition}

\begin{definition}
\label{def:hvi}
The \emph{hypervolume improvement} (\HVI{}) of a set of points $\mathcal Y'$ with respect to an existing Pareto frontier $\mathcal P$ and reference point $\bm r$ is defined as $\HVI{}(\mathcal Y' | \mathcal P, \bm r) = \HV{}(\mathcal P \cup \mathcal Y'| \bm r) -   \HV{}(\mathcal P | \bm r)$.\footnote{Henceforth, we omit $\bm r$ for brevity. As in previous work, we assume that the reference point $\bm r$ is known and specified by the decision maker \citep{yang2019}.}
\vspace{-1.5ex}
\end{definition}

\textbf{Bayesian Optimization (BO)} is a sample-efficient technique for optimizing black-box functions \citep{frazier2018tutorial}. BO relies on a probabilistic surrogate model---typically a Gaussian process (GP), which provides well-calibrated uncertainty estimates---and an acquisition function that leverages the surrogate model to quantify the value of evaluating the objective functions for a design $\bm x$. The acquisition function balances exploring areas with high uncertainty and exploiting regions believed to be optimal.
BO selects the next point $\bm x$ to evaluate by maximizing the acquisition function (which is cheap to evaluate relative to the objective functions), observes a (potentially noisy) measurement of the metrics $\bm y$, adds the new observation to the dataset $\mathcal D = \{(\bm x, \bm y)\} \cup \{(\bm x_i, 
\bm y_i) \}_{i=1}^n $, updates the surrogate model to incorporate the new observation, and repeats this process for a predetermined budget of evaluations. In the MO setting, a common approach is to optimize a random scalarization of the objectives using a single objective acquisition function, such as expected improvement \citep{parego} or Thompson sampling \citep{paria2018flexible}. Alternatively, a multi-objective acquisition function can be used to directly optimize the PF. For example, expected hypervolume improvement (\EHVI{}) \citep{emmerich2006} aims to maximize the \HV{} of the PF under the surrogate model's posterior distribution.

\vspace{-2ex}
\section{Related Work}
Many recent works on BO have considered settings where at the time of implementation either (i) the parameters are subject to noise \citep{Bogunovic2018Robust, oliveira19uncertaininputs, nes}---as we consider in this work---or (ii) the system performance depends on auxiliary unknown environmental variables \citep{kirschner20arobust, iwazaki21ameanvariance}. While previous work has focused on optimizing the expected \citep{toscano-palmerin2018expectation} or the worst-case performance \citep{URREHMAN2014worstcase, sessa2020mixed}, a recent body of work has focused on optimizing more sophisticated risk measures \citep{torossian2020bayesian, Cakmak2020borisk, nguyen2021optimizing,nguyen2021valueatrisk}. However, despite recent significant interest in non-robust MOBO \citep{diversity, suzuki2020multiobjective}, to our knowledge, no prior work has studied robust MOBO.



Motivated by practical limitations due to manufacturing tolerances, \citet{Malkomes2021cas} proposed constraint active search (CAS), which aims to identify diverse solutions in the region of the search space that exceeds a minimum threshold on the objectives. However, CAS does not model or account for input noise, and CAS alone cannot produce any guarantees on robustness to input noise. Methods such as CAS would require a post-hoc analysis using the data collected during optimization to analyze the sensitivity of the solutions to input noise \citep{Calandra2014ParetoFM}. 

Approaching robust design by decoupling data collection and sensitivity analysis is central to the Taguchi method \citep{taguchi1989}. Data acquisition often revolves around finding designs that balance the mean and variance of the sensitive objective under input noise \citep{BEYER2007}. \citet{DO2021robustBO} propose an approach in this vein for the two-objective setting where only one objective is subject to input noise. However, the algorithm does not seek to identify trade-offs with high probability robustness guarantees, and the method does not handle multiple sensitive objectives. In contrast with the Taguchi method, we aim to unify data collection and sensitivity analysis by selecting designs that are believed to yield high-probability performance guarantees.


Outside the BO literature, robust MO optimization has been studied using either EAs \citep{Gupta2005robustconstraintmoo, Zhenan2019robustEA} or assuming access to the explicit mathematical programming formulation of the problem \citep{MAJEWSKI2017robustmoo, ROBERTS2018robustmoo}. Those works have focused on finding the Pareto frontier of the expectation or the worst-case objectives or on finding the Pareto frontier of the nominal objectives with additional constraints on the deviation from the nominal values \citep{deb2005robustEA, avigad2008worstcaseEA}. Some works have considered conceptual properties of different scalarization methods \citep{ide2014}, but not in relation to \mvar{}. EAs that are robust to input noise are not applicable to the small evaluation budget regime that we consider because they typically require a large number of function evaluations \citep{deb2005robustEA}. Even methods that combine EAs with GPs require thousands of evaluations per design \citep{zhou18}.

As a final differentiator from prior work, we consider the practical setting where there are additional black-box constraints that are sensitive to input noise \citep{marzat2013, li2015optimal},
which is a subject addressed by only a few BO methods \citep{Beland2017BayesianOU} even in the single objective case.

\vspace{-2ex}
\section{Multi-Objective Optimization with Noisy Inputs}

In many practical scenarios, the nominal performance of a design can be evaluated by means of a simulation (e.g., by simulating the pharmaceutical process under nominal operating conditions). We consider the setting where we can simulate $\bm f(\bm x)$ for any given design $\bm x \in \mathcal{X}$, but that the design is subject to noise $\bm \xi(\bm x)$ from a known noise process $\bm \xi(\bm x) \sim P(\bm \xi; \bm x)$ at implementation time.\footnote{For brevity, we omit the dependency on $\bm x$ in our notation and write $\bm \xi$ and $P(\bm \xi)$ going forward.}
%
The realized system performance is given by the random variable $\bm f(\xxi)$, where $\xxi$ denotes any known function $g(\bm x,\bm \xi)$ (e.g. for additive noise $\gpertop$ is simply $+$). 
For an extended problem formulation including black-box constraints, see Appendix~\ref{appdx:const}.

In robust optimization, the goal is often to optimize a \emph{risk measure} $\rho[\bm f(\xxi)]$ that maps a random variable to a statistic of its distribution. A common risk measure is the expectation over the input noise distribution \citep{deb2005robustEA,toscano-palmerin2018expectation, nes}, $\mathbb E_{\bm \xi \sim P(\bm\xi)}[\bm f(\xxi)]$, which can be used instead of the random variable $\bm f(\xxi)$ and optimized via standard multi-objective optimization methods. We propose the first MOBO methods for optimizing expectation objectives in Appendix~\ref{appdx:expectation-optimization}.
Despite its widespread use, the expectation risk measure may not always align with the practitioner's true robustness goals. Often, one would prefer solutions with objectives that are better than some performance specification $\bm z \in \mathbb{R}^M$ with high probability (e.g. to maximize production yield) \citep{sarykalin}. Hence, in the single-objective setting, probabilistic risk measures such as value-at-risk (\var{}) are frequently used.

\begin{definition}
\label{def:var}
Given input noise $\bm \xi \sim P(\bm \xi)$ where $\bm \xi \in \mathbb R^d$ and a confidence level $\alpha \in [0,1]$, the value-at-risk for a given point $\bm x$ is:
\begin{equation*}
    \vspace{-0.5ex}
    \var{}_{\alpha}\big[f(\xxi) \big] = \sup \{z \in \mathbb R: P\big[f(\xxi) \geq z\big] \geq \alpha \}.
    \vspace{-1.5ex}
\end{equation*}
\end{definition}

Although several BO methods exist for optimizing \var{} \citep{Cakmak2020borisk, nguyen2021valueatrisk}, they cannot directly be used  in the MO setting because \var{} is not defined for multivariate random variables. A n\"{a}ive way to extend \var{} to the MO setting would be to consider the \var{} of each objective independently. However, this ignores the fact that all $M$ objectives are evaluated at the same realization of $\xxi$.
Considering the \var{} of each objective independently typically leads to overly optimistic risk estimates because objectives under input noise are not typically \emph{simultaneously} greater than or equal to their respective independent \var{}s (i.e. the $(1~-~\alpha)$~-~quantiles) with probability $\geq \alpha$.
Thus it is important to use risk measures such as multivariate value-at-risk (\mvar{}) that account for the joint distribution of the objectives \citep{Prekopa2012MVaR}.

This is illustrated in the center plot in Figure \ref{fig:1d_toy}. Under input noise, the objective values are correlated, which underscores the importance of accounting for the joint distribution of the objectives in measures of robustness. In addition, the center plot in Figure~\ref{fig:1d_toy} shows that $\bm f(\xxi)$ has an asymmetric distribution, even for this very simple and well-behaved toy problem, which highlights how applying \var{} to each objective independently or using the expectation risk measure may conceal underlying variation and risk. Indeed, the results in  Appendix~\ref{appdx:additional_experiments} show that optimizing an expectation risk measure on this problem results in poor performance.


In contrast with \var{} and the expectation risk measure which map a random variable to \emph{single} scalar or vector, \mvar{} maps a random variable to a non-dominated \emph{set} of vectors in the outcome space that are dominated by $\alpha$-fraction of all possible realizations, where $\alpha \in [0,1]$ is a hyperparameter set by the practitioner.  That is, each vector in the \mvar{} set corresponds to an objective specification that a design will meet with probability $\geq\alpha$. Therefore, $\alpha$ is an interpretable risk level that can be valuable in manufacturing applications where one wishes to find the PF of all objective specifications with a guaranteed yield fraction ($\alpha$).
\vspace{-2ex}
\begin{definition}
\label{def:mvar}
The $\mvar$ of $\bm f$ for a given point $\bm x$ and confidence level $\alpha \in [0,1]$ is:
\begin{equation*}
\begin{aligned}
    \mvar{}_\alpha\big[\bm f(\xxi&) \big] =\\ \textsc{Pareto}&\big(\big\{ \bm z \in \mathbb R^M : P\big[\bm f(\xxi) \geq \bm z\big] \geq \alpha\big\}\big).
\end{aligned}
\end{equation*}
\end{definition}
\vspace{-1ex}
The \mvar{} set over $X$ specifies objective vectors~$\bm z$ such that there exists a known design $\bm x \in X$ with corresponding random objectives $\bm f(\xxi)$ under $P(\bm \xi)$ that dominate $\bm z$ with probability $\geq \alpha$.
\begin{definition}
\label{def:global_mvar}
The $\mvar$ for a set of points $X$ is:
\begin{equation*}
\begin{aligned}
    \mvar{}_\alpha\big[\{\bm f(\xxi&)\}_{\bm x \in X} \big] =\\
    &\textsc{Pareto}\bigg( \bigcup_{\bm x \in X} \mvar{}_\alpha\big[\bm f(\xxi) \big]\bigg).
\end{aligned}
\end{equation*}
\end{definition}
\vspace{-2ex}
The global \mvar{} across the design space, $\mvar{}_\alpha\big[\{\bm f (\xxi)\}_{\bm x \in \mathcal X}\big]$, is a robust analogue of the PF in the standard MO setting. The concept of the \mvar{} of a set of design points $X$ is a novel contribution of this work. 

\textbf{Optimization Goal} In this work, our goal is to identify the \mvar{} set across the design space: $\mvar{}_\alpha\big[\{\bm f (\xxi)\}_{\bm x \in \mathcal X}\big]$. Given an approximate \mvar{} set across the design space, a decision-maker can pick a design according to their preferences. Similar to the standard MO setting, the \HV{} of the \mvar{} set across the design space can be used to evaluate optimization performance. 

\begin{figure*}[ht!]
    \centering
    \includegraphics[width=\linewidth]{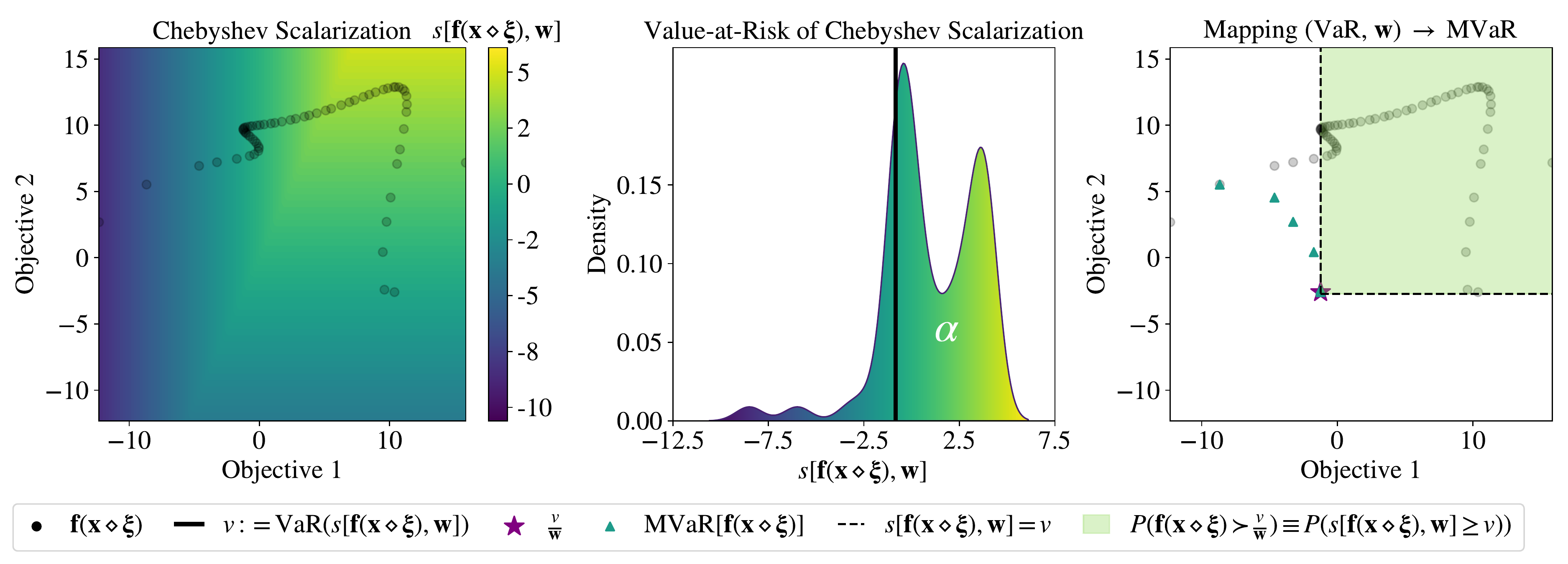}
    \vspace{-3.5ex}
    \caption{\label{fig:bijection} Construction of \mvar{} sets via random scalarizations for the 1-d example in Figure~\ref{fig:1d_toy}. Left: The function values for a single design with zero-mean Gaussian input perturbations with a standard deviation of 0.1 are marked by black points. The background is a contour showing the value of a Chebyshev scalarization across the objective space. Center: The probability density of a Chebyshev scalarization over the function values under input noise and the value-at-risk $v$ of a Chebyshev scalarization for $\alpha=0.9$. The probability mass to the right of the black line is equal to $\alpha$. Right: Leveraging Theoream~\ref{thm:bijection}, $\var{}, \bm w$ can be mapped to point in the \mvar{} set that is dominated by the objectives under input perturbations 90\% of the time. The green triangles represent a discrete approximation of the \mvar{} set with 64 samples from the input noise distribution.
    The green area indicates the region that dominates the identified \mvar{} point, or, equivalently, the area for which the Chebyshev scalarization defined by $\bm w$ is greater than $v$.}
\vspace{-2ex}
\end{figure*}
\vspace{-2ex}
\section{Optimizing \mvar{}}
\label{sec:opt_mvar}

A natural approach for optimizing \mvar{} is to directly maximize the \HV{} dominated by the \mvar{} set. 
Although \mvar{} of a given point typically cannot be evaluated directly,
it can be approximated using $n_{\bm{\xi}}$ MC samples of $\bm \xi$, provided that independent samples can be draw from the noise process.
Thus, evaluating the \mvar{} set across the previously evaluated designs using the surrogate requires sampling from the posterior of $P(\bm f| \mathcal D)$ evaluated jointly at $\gpert{\bm x_1}{\bm\xi_i}, \dotsc, \gpert{\bm x_n}{\bm\xi_i}$ for $i=1, ..., n_{\bm{\xi}}$, where $X_{1:n} := \{\bm x_1, ..., \bm x_n\}$ are the previously evaluated designs.
Since $\{\bm f(\xpxi)\}_{\bm x' \in X_{1:n}}$ is typically not observed,
the corresponding posterior predictions may have large uncertainties. 
In order to get a reliable estimate of \mvar{},
we would need to integrate over the posterior distribution of $\{\bm f(\xpxi)\}_{\bm x' \in X_{1:n}}$. \qNEHVI{} \citep{daulton2021parallel} is a variant of \EHVI{} that integrates over the uncertainty in function values at previously evaluated designs. This makes \qNEHVI{} suitable for optimizing \mvar{}. 

However, several computational issues---including time complexity that is exponential in the number of objectives and exponential in the size of $\mvar{}_\alpha\big[\bm f(\xxi) \big]$---make it infeasible to directly optimize \mvar{} with \qNEHVI{} in many settings. We defer a detailed discussion to Appendix~\ref{appdx:direct-mvar-opt} and present an empirical evaluation in Appendix~\ref{appdx:additional_experiments}.
\vspace{-1ex}
\subsection{Relationship between \mvar{} and 
Scalarizations}
\label{subsec:theory_cheby}
An alternative to direct optimization of the \mvar{} set is to apply a scalarization to the objectives and use a standard risk measure on the scalarized objective. Unlike the use of independent risk measures on each objective, this approach accounts for the correlation between outcomes induced by the input perturbation. 
In this section, we present our main theoretical result: under limited assumptions, there exists a bijection, based on \var{}, that maps a particular family of scalarizations---Chebyshev scalarizations \citep{Kais99}---to points in the \mvar{} set. In other words, each point in the \mvar{} set corresponds to a particular set of scalarization weights. This means that we can recover the entire \mvar{} set using these scalarizations, without any loss. Proofs and additional theoretical results including extensions to the constrained setting are provided in Appendix~\ref{appdx:proofs}.
\begin{definition}
\label{def:chebyshev}
Let $\bm w \in \Delta_+^{M-1}$, where $\Delta_+^{M-1}$ denotes the positive $(M-1)$-simplex, and let $\bm r \in \mathbb R^M$. The Chebyshev scalarization $s[\bm y, \bm w, \bm r]$ is given by $s[\bm y, \bm w, \bm r] = \min_i w_i(y_i-r_i)$, where $\cdot_i$ denotes the $i^{th}$ dimension.\footnote{Typically, $\bm f$ is scaled to the unit cube using the $\bm r$ as the lower bound before applying the scalarization. Since the scaled reference point is~$\bm 0$, hence forth, we omit $\bm r$ for brevity. See Appendix~\ref{appdx:method_details} for details.} 
\vspace{-0.5ex}
\end{definition}

The contour in the left plot in Figure~\ref{fig:bijection} shows the Chebyshev scalarization for a fixed $\bm w$ for the two-objective toy problem from Figure~\ref{fig:1d_toy} and illustrates a connection between Pareto dominance and the Chebyshev scalarization, which we formalize below. The black points are function values under sampled perturbations for a single design $\bm x$. The center plot in Figure~\ref{fig:bijection} shows the distribution of Chebyshev scalarization values for a given $\bm w$ and the black line indicates the $\alpha$-level \var{}. The right plot in Figure~\ref{fig:bijection} illustrates that using the \var{} of a Chebyshev scalarization, we can deduce a point $\bm z$
such that the function values under the input perturbations will dominate 
$\bm z$ 
with probability $\geq\alpha$.
\begin{restatable}[\var{} of Chebyshev scalarization $\Rightarrow$ Pareto Dominance]{lemma}{dominancecdf}
\label{lemma:dominance_cdf}
Let 
$v = \var{}_{\alpha}\big(s[\bm f(\xxi), \bm w]\big)$.
Then, $P\big[\bm f(\xxi) \geq \frac{v}{\bm w}\big] \geq \alpha$, where $\frac{v}{\bm w}$ denotes element-wise division.
\end{restatable} 
\vspace{-1ex}
The condition in Lemma~\ref{lemma:dominance_cdf} is one criterion for membership in the \mvar{} set (the other being Pareto efficiency). The shaded region in right plot in Figure~\ref{fig:bijection} illustrates the region that dominates $\bm z$. With probability $\geq \alpha$, $\bm f(\xxi)$ will fall within the shaded region. \emph{This result enables translating the \var{} of a Chebyshev scalarization into an interpretable, high-probability guarantee on robust performance in terms of Pareto dominance.}



\begin{assumption}
\label{con:cont_cdf} $\bm f(\xxi)$ has a continuous, strictly increasing CDF $F$. I.e., if $\bm f(\bm x) \succ \bm f(\bm x')$, then  $F\big[\bm f(\bm x )\big] > F\big[\bm f(\bm x')\big]$.\footnote{Assumption~\ref{con:cont_cdf} holds when $\bm f$ is a function sampled from a GP prior with many commonly used covariance functions. See Lemma~\ref{lemma:cont_rv} for a formal statement and Appendix~\ref{appdx:gp_cdfs} for proof.}
\end{assumption}

If Assumption~\ref{con:cont_cdf} is met, then for any $\bm w$, $\frac{v}{\bm w}$ is not dominated by any other point in the \mvar{} set, and hence, $\frac{v}{\bm w}$ is an element of the \mvar{} set. 
%
Furthermore, we 
have the following: 
\vspace{-0.1ex}
\begin{restatable}[$\mvar{} \iff \var{}$ of Chebyshev  scalarization] {theorem}{bijection}
\label{thm:bijection}
Given $\bm x, \bm f,$ and $P(\bm \xi)$, let $h: \mvar{}_\alpha\big[\bm f(\xxi)\big] \rightarrow \Delta_+^{M-1}$ be the function $h(\bm z) = \bm w = \frac{1}{\bm z||\frac{1}{\bm z}||}$. Under Assumption~\ref{con:cont_cdf}, $h(\cdot)$ is bijective and $h^{-1}(\bm w) = \bm z = \frac{1}{\bm w}\var{}_\alpha\big(s[\bm f(\xxi), \bm w]\big)$. 
\end{restatable}
\vspace{-0.1ex}
Theorem~\ref{thm:bijection} \emph{provides a technique for generating points in the \mvar{} set using the \var{} of Chebyshev scalarizations with different weights.}
Importantly, any design that is globally optimal with respect to \mvar{} is a maximizer of the \var{} of a Chebyshev scalarization.
This naturally motivates a methodology for identifying the global \mvar{} set by optimizing the \var{} of random Chebyshev scalarizations.
\vspace{-3ex}
\begin{corollary}[Consistent Optimizers]
\label{cor:consistent_optimizers} Suppose $\bm z$ is a point in the global \mvar{} set, i.e., $\bm z \in \mvar{}_\alpha\big[\{\bm f(\xxi)\}_{\bm x \in \mathcal X}\big]$. Let $\mathcal X_{\bm{z}}^*$ be the set of designs such that for all $\bm x \in \mathcal X_{\bm{z}}^*$, $\bm z \in \mvar{}_\alpha\big[\bm f(\xxi)\big]$.
Then every $\bm x \in \mathcal X_{\bm{z}}^*$ is a maximizer of $\var{}_\alpha\big(s[\bm f(\xxi), \bm w]\big)$ for $\bm w = \big(\bm z||\frac{1}{\bm z}||\big)^{-1}$.
\end{corollary}

\vspace{-1.5ex}
\subsection{\mars{}: \mvar{} Approximation via Random Scalarizations}
\label{subsec:random_scalar}
The connection between the \var{} of a Chebyshev scalarization and \mvar{} can be exploited to optimize the global \mvar{} by randomly sampling a Chebyshev scalarization at each BO iteration\footnote{We sample weights uniformly from $\Delta_+^{M-1}$, which we find works well empirically.} and optimizing the \var{} of the Chebyshev scalarization using a single objective BO algorithm, such as Noisy Expected Improvement (\nei{}) \citep{letham2019noisyei}---which is required (rather than expected improvement) since we must integrate over the unknown best unknown incumbent value---or Thompson sampling (TS) \citep{thompson, paria2018flexible}. We refer to this technique as \mvar{} Approximation via Random Scalarizations (\mars{}). \mars{} is a simple, theoretically-grounded technique, and we find, in Section~\ref{sec:experiments}, that \mars{} performs well empirically. In the main text, we focus on \mars{} with \nei{} (denoted as \marsnei{}) , but we derive and empirically evaluate upper confidence bound (UCB) \citep{ucb} and TS variants in Appendices~\ref{appdx:mars_acqfs} and \ref{appdx:experiments}.
    
For \marsnei{}, we use the MC formulation of \nei{} \citep{balandat2020botorch} so that the Chebyshev scalarizations can be computed in the \nei{} as composite objectives \citep{astudillo2019composite}. We optimize the the acquisition function using sample-path gradients that leverage well-studied gradient estimators of \var{} (see Appendix~\ref{appdx:gradients-of-var} for details). 
See Appendix~\ref{appdx:method_details} for details on optimization.

\vspace{-1.5ex}
\section{Experiments}
\label{sec:experiments}
\begin{figure*}[ht]
\def\mtplotwidth{0.245\linewidth}
    \centering
    \includegraphics[width=\linewidth]{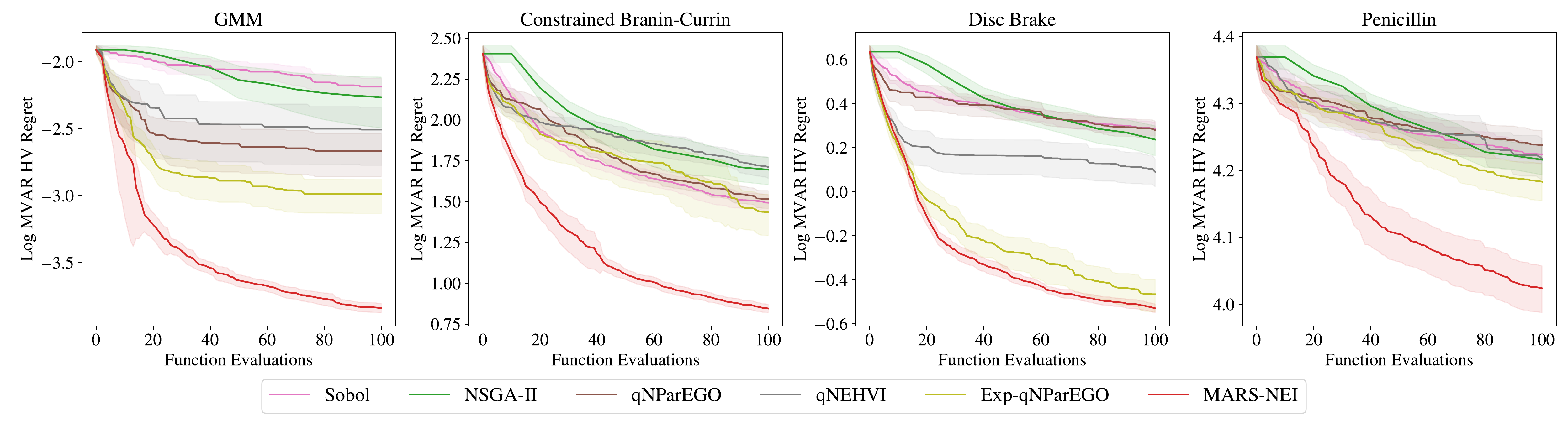}
    \vspace{-4ex}
    \caption{The log \mvar{} \HV{} regret after the initial space-filling design. For each method, we plot the mean and 2 standard errors of the mean over 20 trials.}
    \label{fig:results}
    \vspace{-4ex}
\end{figure*}


\begin{figure}[ht]
\def\yieldplotwidth{0.66\linewidth}
\centering
    \includegraphics[width=\linewidth]{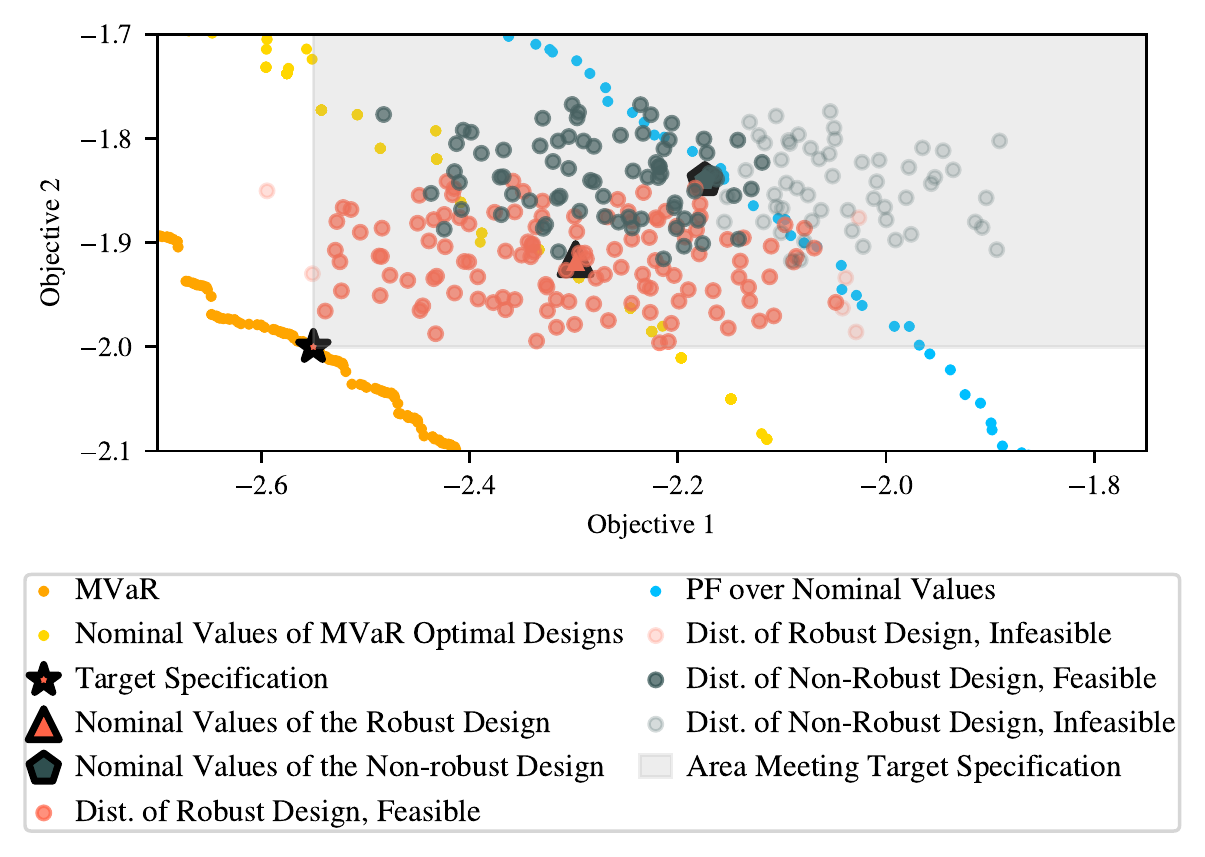}
    \vspace{-4ex}
    \caption{The yield from selecting a robust versus a non-robust design on the Disc Brake problem. Although the non-robust design is feasible under the nominal objectives, it is located near the boundary of the feasible region in design space and violates some of the black-box constraints (not shown) under a large fraction of input perturbations.
    }
    \label{fig:yield}
    \vspace{-3.5ex}
\end{figure}

In this section, we provide an empirical demonstration of robust MOBO on synthetic and real-world problems. We compare three broad classes of methods: (i) Non-robust methods including NSGA-II \citep{deb02nsgaii}, \qNParego{} \citep{daulton2020ehvi}, \qNEHVI{} \citep{daulton2021multiobjective}, (ii) methods for optimizing expectation risk measures, e.g. using \qNParego{} (denoted as \textsc{Exp-}\qNParego{}), and (iii) methods for optimizing \mvar{} via \mars{}. For readability, we only include one expectation and one \mvar{} optimization method in the main text, both based on \nei{}. In Appendix \ref{appdx:experiments}, we evaluate additional methods including \mars{} with TS and UCB and methods for direct \mvar{} optimization based on \textsc{NEHVI}. We consider the expectation risk measure because it is simple and performs well in many scenarios.
All robust (expectation and \mvar{}) methods are our novel contributions. Additionally, we compare against a quasi-random policy, which selects the designs to evaluate according to
a scrambled Sobol sequence \citep{owen2003quasi}. 

For all BO-methods, we begin by evaluating $2(d+1)$ design points from a scrambled Sobol sequence. We use $n_{\bm \xi}= 32$ samples because we find that setting $n_{\bm \xi} > 32$ yields little-to-no improvement in optimization performance (see Figure~\ref{fig:n_xi_plots}). See Appendix~\ref{appdx:experiments} for details on all methods. Our code is open-sourced at \url{github.com/facebookresearch/robust_mobo}.
\vspace{-2ex}
\subsection{Synthetic Problems}
\textbf{Gaussian Mixture Model (GMM)} ($d=2,~M=2,~ \alpha=0.9$): This is a variant of the GMM problem from \citet{nes} where each objective is an independent GMM. We use a multiplicative noise model, i.e., $\xxi := \bm x \bm \xi$, where $\bm \xi \sim \mathcal{N}(\mu=\bm 1, \Sigma = 0.07 I_2)$
with $I_n$ denoting the $n$-dimensional identity matrix.
In Appendix~\ref{appdx:gmm_noise_level}, we present multiple variations of this problem to demonstrate the consistency of our methods under different noise models.

\textbf{Constrained Branin Currin} ($d=2,~M=2,~ V=1, ~\alpha=0.7$): 
We subject this problem, which originates from \citet{daulton2020ehvi},
to a heteroskedastic input noise process given by
$P(\bm \xi; \bm x) = \mathcal{N}\left(\mu=\bm 0,\Sigma=0.05\left(1 + \sigma\left(1 - 2x_0 \right) \right) I_2 \right)$, where $\sigma(x) = \frac{1}{1 + e^{-x}}$.
The optimal designs with respect to the nominal objectives are on the boundary of the feasible region with respect to the outcome constraint. Hence, the optimal designs with respect to the nominal metrics often violate the constraint under input perturbations.
\vspace{-1ex}
\subsection{Real-World Problems}
\textbf{Disc Brake} ($ d=4,~M=2,~ V=4,~ \alpha=0.95$): In this disc brake manufacturing problem, the goal is to minimize the brake's mass and the stopping time of a vehicle by tuning the inner and outer radii of the disc, the engaging force, and the number of friction surfaces \citep{ray2002}. Following \citet{emch1994}, we use zero-mean uniform input noise with a maximum absolute perturbation value of 5\% of the range of each parameter (except for the number of friction surfaces, which is noise-free).

\textbf{Penicillin Production} ($d=7,~ M=3,~\alpha=0.8$): This problem considers optimizing the manufacturing process of penicillin \citep{liang2021scalable}. The goal is to maximize the
yield while minimizing time-to-ferment and CO$_2$ output by tuning 7 initial conditions of the chemical reaction. Each parameter is subject to independent zero-mean Gaussian noise, where the standard deviation ranges from 0.5\% to 3\% of the parameter's domain (see Appendix~\ref{appdx:experiments} for details).
\vspace{-1.5ex}
\subsection{Results} In Figure~\ref{fig:mvar-results}, we evaluate all methods in terms of log \HV{} regret, which is the difference in \HV{} between the true global \mvar{} set and the \mvar{} of the set of designs evaluated by each method (see Appendix~\ref{appdx:evaluation_details} for details on estimating true global \mvar{}).
Figure~\ref{fig:results} shows that \mars{} is consistently the best performing method. The non-robust methods consistently perform poorly on all problems. On the GMM problem,  \textsc{Exp-}\qNParego{} performs worse than \marsnei{} because the optimal designs under expectation risk measure are in a disjoint part of the search space from the \mvar{}-optimal designs. On the Constrained Branin Currin problem, the nominal methods perform no better than quasi-random Sobol search. On the penicillin problem, \mars{} vastly outperforms all other methods. 


Although the log \HV{} regret highlights the performance of \mars{}, it does not fully capture the \emph{necessity} of using a robust method in practice. In Figure~\ref{fig:yield}, we analyze the yield (i.e. the probability of the objectives exceeding a performance specification under the input noise distribution) of a robust and a non-robust design on the Disc Brake problem. Using a target performance specification chosen from the \mvar{} set, we see that if a decision-maker were to select a non-robust design (green pentagon) that is optimal with respect to the nominal objectives and nominal values that meet the target specification, the yield would only be $58.2\%$.
In contrast, an \mvar{}-optimal solution can be chosen such that it meets target specification with high probability. For example, the robust design marked by the orange triangle enjoys a yield of $95.3\%$.
As shown in Figure~\ref{fig:yield}, this is because the non-robust design often does not satisfy all of the black-box constraints under input perturbations. In this problem, the objectives are relatively robust to noise (much more so than the toy example in Figure \ref{fig:1d_toy}), but the feasibility of a design (and therefore the yield) is highly sensitive to input noise when the design is near the boundary of the feasible region in design space. 


Table~\ref{table:wall-clock-time} reports the wall times for running a single iteration of BO with each algorithm. Not only is \marsnei{} computationally tractable on all problems, but it achieves wall times that are competitive to alternative algorithms. 

Evaluation in additional problem settings and comparisons against additional methods in Appendix~\ref{appdx:additional_experiments} further validate that \marsnei{} is consistently a top performer and yields competitive wall times. We find that \marsts{} performs slightly worse, on average, in terms of log \HV{} regret than \marsnei{}, but that it has shorter wall times. Methods for direct \mvar{} optimization with \qNEHVI{} perform comparably to \marsnei{}, but the direct \mvar{} optimization is prohibitively expensive in terms of wall time and memory requirements and was infeasible to run on many problems (including nearly all problems with $>2$ objectives). In contrast, Figure~\ref{fig:gmm-results_3obj_4obj} shows that \marsnei{} can scale to problems with $> 2$ objectives and consistently performs best in those settings. Additionally, in Appendix~\ref{appdx:experiments}, 
we show that \marsnei{} works well under a wide variety of input noise processes, scales well with increasing batch sizes (in the parallel evaluation setting), and is not sensitive to the number of MC samples $n_{\bm \xi}$ used to estimate $P(\bm \xi)$, for $n_{\bm \xi}\geq 32$.

\vspace{-1.5ex}
\section{Discussion}
In this work, we formulate the goal of MO robust optimization under input noise as optimizing the \emph{global} \mvar{} set---a novel concept that is a robust analogue of the 
Pareto frontier
in the standard MO setting. We derive a 
correspondence between \mvar{} and Chebyshev scalarizations based on \var{}. This theoretical result naturally motivates a computationally efficient approach (\mars{}) for using BO to optimize the global \mvar{} set with high sample-efficiency. Empirically, we find that \mars{} consistently outperforms alternative approaches and achieves competitive wall times. 

Although our focus has been on the small evaluation budget regime and BO, our theoretical results are far more general. 
The connection between \mvar{} and Chebyshev scalarizations could be leveraged by  gradient-based  and evolutionary methods 
to scale global \mvar{} optimization to settings with large evaluation budgets. 
We hope that our contributions serve as a foundation for future advances in methods for robust MO optimization under input noise.

\section*{Acknowledgements}
We thank Ben Letham, David Eriksson, James Wilson, Martin J\o{}rgensen, and Raul Astudillo, as well as the members of the Oxford Machine Learning Research Group, for providing insightful feedback. In addition, Sait Cakmak and Enlu Zhou are grateful for support by the Air Force Office of Scientific Research under Grant FA9550-19-1-0283. 

\bibliographystyle{plainnat}
\bibliography{ref}
\FloatBarrier
\clearpage
\newpage
\onecolumn
\appendix
\section{Theory and Proofs}
\label{appdx:proofs}

\subsection{Preliminaries}
\begin{definition}
\label{def:weak_pareto}
Let $\mathcal F = \{\bm f(\bm x) : \bm x \in X \subseteq \mathcal X\}$. The weakly efficient Pareto frontier is $\textsc{WeakPareto}(\mathcal F) = \{\bm f(\bm x) \in \mathcal F~ : ~\nexists ~\bm x' \in X ~s.t.~  \bm f(\bm x') > \bm f(\bm x)\},$
and $\textsc{Pareto}(\mathcal F) \subseteq \textsc{WeakPareto}(\mathcal F)$. If there are constraints $\bm c(\bm x)$, elements of \textsc{WeakPareto} are subject to the additional membership Assumption that $\bm c(\bm x) \geq \bm 0$. We call the corresponding set of optimal designs the \emph{weak Pareto set}.
\end{definition}

\begin{definition}
\weakmvar{} is defined in the same way as \mvar{}, but only requires that its elements are weakly Pareto optimal.
\end{definition}

\begin{definition}[\citet{Prekopa2012MVaR}]
If Assumption~\ref{con:cont_cdf} holds, then we can express $\mvar{}$ with an equality with respect to $\alpha$:
\begin{equation}
\begin{aligned}
    \label{eqn:mvar_alpha}
    \mvar{}_\alpha&\big[\bm f(\xxi) \big] = \big\{ \bm z \in \mathbb R^M : P\big[\bm f(\xxi) \geq \bm z\big] = \alpha \big\}.
\end{aligned}
\end{equation}
\end{definition}
\subsection{Proofs}
\label{appdx:proofs_sub}
\begin{restatable}{lemma}{boolequivalence}
\label{lemma:bool_equivalence}
Let $\bm y \in \mathbb R^M$ and $v \in \mathbb R$. Then $s[\bm y, \bm w] \geq v \iff \bm y \geq \frac{v}{\bm w}.$
\end{restatable} 

\begin{proof}
    This follows directly from Definition~\ref{def:chebyshev}.
    \begin{align*}
        s[\bm y, \bm w] \geq v &\iff \min_i w_i y_i \geq v \\
        &\iff w_i y_i \geq v ~\forall i \\
        &\iff y_i \geq \frac{v}{w_i} ~\forall i \\
        &\iff \bm y \geq \frac{v}{\bm w}.
    \end{align*}
\end{proof}
Lemma~\ref{lemma:bool_equivalence} states that a lower bound $v$ on the value of a Chebyshev scalarization of an objective vector $\bm y$ can be used to define a point, $\frac{v}{\bm w}$, that $\bm y$ is greater than or equal to. We can extend this to make a similar statement about the \var{} of a Chebyshev scalarization.
\dominancecdf*
\begin{proof}
From Definition~\ref{def:var}, we have that 
\begin{equation*}
    \var{}_{\alpha}\big[f(\xxi) \big] = \sup \{z \in \mathbb R: P\big[f(\xxi) \geq z\big] \geq \alpha \}.
\end{equation*}
Hence, $P(s[\bm f(\xxi), \bm w] \geq v) \geq \alpha$. From Definition~\ref{def:chebyshev},
 we have $P(\min_i [ w_if_i(\xxi)] \geq v) \geq \alpha.$
By Lemma~\ref{lemma:bool_equivalence}, the statement $\min_i [w_if_i(\xxi)] \geq v$ is equivalent to $\bm f(\xxi) \geq \frac{v}{\bm w}.$ Hence,
$P\big(\bm f(\xxi) \geq \frac{v}{\bm w}\big) \geq \alpha.$
\end{proof}

\begin{restatable}[\var{} of Chebyshev scalarization $\Rightarrow$ \weakmvar{}] {lemma}{varinweakmvar}
\label{lemma:var_in_weak_mvar}
Let $v = \var{}_\alpha\big(s[\bm f(\xxi), \bm w]\big)$. Then, $\frac{v}{\bm w} \in \weakmvar{}_\alpha\big[\bm f(\xxi)\big]$.
\end{restatable}
\begin{proof}
Let $\bm z = \frac{v}{\bm w}$. Suppose there exists $\bm z' \in \weakmvar{}_\alpha(\bm f(\xxi))$ such that $\bm z' > \bm z$. Since $\bm z' \in \weakmvar{}_\alpha(\bm f(\xxi))$, we have that $P(\bm f(\xxi) \geq \bm z') \geq \alpha$. Note that
$\bm f(\xxi) \geq \bm z'$ implies that $\min_iw_if_i(\xxi) \geq \min_i w_iz_i'$. Hence, $P(\bm f(\xxi) \geq \bm z') \geq \alpha$ implies that $P( s[\bm f(\xxi), \bm w] \geq s[\bm  z', \bm w]) \geq \alpha$. Since $\bm z' > \bm z$ and $\bm w \in \Delta_+^{M-1}$, we have that $s[\bm  z', \bm w] > s[\bm  z, \bm w] = v$. But 
this contradicts Definition~\ref{def:var}. Since there does not exist $\bm z' \in \weakmvar{}_\alpha\big[\bm f(\xxi)\big]$ such that $\bm z' > \bm z$, we have that $\bm z \in \weakmvar{}_\alpha\big[\bm f(\xxi)\big]$.
\end{proof}

If Assumption~\ref{con:cont_cdf} holds, the inequalities with respect to $\alpha$ in Lemma~\ref{lemma:dominance_cdf} become equalities, and we show that $\frac{v}{\bm w}$ is strictly Pareto optimal.
\begin{restatable}[\var{} of Chebyshev scalarization $\Rightarrow$ \mvar{}] {lemma}{varinmvar}
\label{lemma:var_in_mvar}
Let $v = \var{}_\alpha\big[s[\bm f(\xxi), \bm w]\big]$. If Assumption~\ref{con:cont_cdf} holds, then $\frac{v}{\bm w} \in \mvar{}_\alpha\big[\bm f(\xxi)\big]$.
\end{restatable}
\begin{proof}
Let $\bm z := \frac{v}{\bm w}$.
Suppose there exists $\bm z' \in \mvar{}_\alpha\big[\bm f(\xxi)\big]$ such that $\bm z' \succ \bm z$. Because $F(\cdot)$ is a strictly increasing CDF, $P(\bm f(\xxi) \geq \bm z') = F(\bm z') > F(\bm z) = P(\bm f(\xxi) \geq \bm z)$. From Lemma~\ref{lemma:dominance_cdf}, we have that $P\big(\bm f(\xxi) \geq \bm z \big) \geq\alpha$. 
Because $\bm z' \in \mvar{}_\alpha\big[
\bm f(\xxi)\big]$ from Equation~\eqref{eqn:mvar_alpha}, we have that $P\big(\bm f(\xxi) \geq \bm z'\big) = \alpha$. Hence, $P\big(\bm f(\xxi) \geq \bm z \big) \geq P\big(\bm f(\xxi) \geq \bm z' \big).$ This is a contradiction.
\end{proof}
Lemma~\ref{lemma:var_in_mvar} provides a technique for generating points in the \mvar{} set using the \var{} of Chebyshev scalarizations with different $\bm w$.

Now, consider the reverse mapping. 
\begin{restatable}[\mvar{} $\Rightarrow$ \var{} of Chebyshev scalarization]{lemma}{mvartovar}
\label{lemma:mvar_to_var}
Suppose that $\bm z \in \mvar{}_\alpha\big[\bm f(\xxi)\big]$. Then $\bm z = \frac{v}{\bm w}$ for $\bm w: = \frac{1}{\bm z}||\frac{1}{\bm z}||_1^{-1} \in \Delta^M_+$ and $v=\var{}_\alpha\big(s[\bm f(\xxi), \bm w]\big).$
\end{restatable}
\begin{proof}
Without loss of generality, assume that $\bm z > \bm 0$.\footnote{\label{footnote:pos_z}This can easily be guaranteed by shifting the objectives to be strictly positive. If the objectives are not unbounded from below, the reference point, which is commonly supplied by the decision-maker in multi-objective optimization effectively provides a lower bound on the objectives.} 
Let $v = ||\frac{1}{\bm z}||^{-1} \in \mathbb R_+$. Then, $\bm z = \frac{v}{\bm w}$. Hence, all we need to show is that $v$ equals $\var{}_\alpha\big(s[\bm f(\xxi), \bm w]\big)$. Let us define $v' := \var{}_\alpha\big(s[\bm f(\xxi), \bm w]\big).$
By definition, 
\begin{equation*}
    v' = \sup \{v'' \in \mathbb R: P\big(s[\bm f(\xxi), \bm w] \geq v''\big] \geq \alpha \}.
\end{equation*}
Since $\bm z \in \mvar{}_\alpha\big[\bm f(\xxi)\big]$, $P\big[\bm f(\xxi) \geq \bm z\big] \geq \alpha$. Hence, $P\big[\bm f(\xxi) \geq \frac{v}{\bm w}\big] \geq \alpha$. Using Lemma~\ref{lemma:bool_equivalence}, we have that $P\big(s[\bm f(\xxi), \bm w] \geq v\bigr) \geq \alpha$. Since $v'$ is the supremum, we have that $v \leq v'$. 
Suppose now that $v < v'$. Note that 
$\frac{v}{\bm w} = \bm z \in \mvar{}_\alpha\big[\bm f(\xxi)\big].$ From Lemma~\ref{lemma:var_in_weak_mvar}, $\frac{v'}{\bm w} \in \weakmvar{}_\alpha\big(\bm f(\xxi)\big)$. Since $v < v'$, 
$\bm z = \frac{v}{\bm w} \prec \frac{v'}{\bm w}.$ Hence, $\bm z$ cannot be in $\mvar{}_\alpha\big[\bm f(\xxi)\big]$ because it is dominated by $\frac{v'}{\bm w}$ and $\frac{v'}{\bm w} \in \weakmvar_\alpha\big(\bm f(\xxi)\big)$.
This is a contradiction. Hence $v\geq v'$ and therefore it follows that $v=v' = \var{}_\alpha\big(s[\bm f(\xxi), \bm w]\big).$
\end{proof}

Let $h:\mvar{}_\alpha\big[\bm f(\xxi)\big] \rightarrow \Delta_+^{M-1}$ be given by $h(\bm z) = \frac{1}{\bm z||\frac{1}{\bm z}||}$. Being the element-wise application of a scalar injective mapping ($z \mapsto 1/z)$, it is clear that $h$ is injective. However, $h(\cdot)$ is not necessarily bijective, since without Assumption~\ref{con:cont_cdf} there may be weight vectors $\bm w$ such that $\nexists~ \bm z \in \mvar{}_\alpha\big[\bm f(\xxi)\big] ~ \text{ s.t }.~h(\bm z) = \bm w$. 

\bijection*
\begin{proof}
This follows directly from Lemma~\ref{lemma:var_in_mvar} and Lemma~\ref{lemma:mvar_to_var}.
\end{proof}

\begin{corollary}[\mvar{} via Scalarization]
\label{cor:bijection_scalarization_theory}
\weakmvar{} enjoys the following scalarization representation:
$\weakmvar{}_\alpha\big(\bm f(\xxi)\big) =  \{\frac{1}{\bm w}\var{}_\alpha\big(s[\bm f(\xxi), \bm w]\big) ~:~ \bm w \sim \Delta_+^{M-1}\}$.
If Assumption~\ref{con:cont_cdf} holds, then
$\mvar{}_\alpha\big[\bm f(\xxi)\big] =  \{\frac{1}{\bm w}\var{}_\alpha\big(s[\bm f(\xxi), \bm w]\big) ~:~ \bm w \sim \Delta_+^{M-1}\}$.
\end{corollary}
Although the \mvar{} representation in Corollary~\ref{cor:bijection_scalarization_theory} depends on Assumption~\ref{con:cont_cdf}, \emph{Lemma~\ref{lemma:dominance_cdf} recovers all weakly Pareto optimal points even if this assumption is not met} because $\mvar{} \subseteq \weakmvar{}$. Hence, with or without Assumption~\ref{con:cont_cdf}, Theorem~\ref{cor:bijection_scalarization_theory} can be used to approximate the \mvar{} set.
\begin{result}
[\mvar{} Approximation]
\label{res:mvar_approximation}
\mvar{} can be approximated with a finite set of weight vectors: $\reallywidehat{\mvar{}}\big(\bm f(\xxi)\big) = \{\frac{1}{\bm w_i}\var{}_\alpha\big(s[\bm f(\xxi), \bm w_i]\big)\}_{i=1}^{N_{\mvar{}}}$.
\end{result}

\subsection{Discussion of the Assumption of Continuous, Strictly-Increasing CDFs with Gaussian Processes}
\label{appdx:gp_cdfs} 
In Bayesian Optimization with Gaussian Process surrogates, it is assumed that the objective function $\bm f$ is sample path from a Gaussian process prior. In this setting, the random variable $\bm f(\xxi)$ for a deterministic sample path $\bm f$ (where the stochasticity comes solely from $\bm \xi$) has a continuous, strictly increasing CDF $F(\cdot)$ for many commonly used covariance functions. Hence, the bijective relationship in Theorem~\ref{thm:bijection} holds given a suitable choice of covariance function.\footnote{Using a discrete MC approximation of the input noise distribution means that the CDF is not strictly increasing.}

\begin{restatable}{lemma}{contrv}
\label{lemma:cont_rv}
If $\bm f$ is a sample path from a multi-output Gaussian Process prior where the  covariance function and the covariance function of the derivative process are strictly positive definite and with sample paths that are differentiable,\footnote{\label{footnote:diff_sample_paths} Sample paths of many commonly used covariance functions are differentiable \citep{paci}.} $\mathcal X$ is a compact set, and $P(\xxi)$ is a continuous distribution with strictly positive support, then Assumption~\ref{con:cont_cdf} holds.
\end{restatable}
\begin{proof}
Consider the case of a scalar function $f$. 
By Lemma 1 of \citet{Cakmak2020borisk}, the density of $f(\xxi)$ is strictly positive. Hence, any interval with positive Lebesgue measure has non-zero density. So the cumulative density function of $f(\xxi)$ is strictly increasing.
Consider the case of a multi-output sample path $\bm f$. Suppose that the joint CDF is not strictly increasing. Then there exist $\bm y, \bm y'$ such that $\bm y \geq \bm y'$ and there exists at least one $i \in \{1,..., M\} ~s.t.~ y_i > y'_i$ and $F(\bm y) \leq F(\bm y').$ Since $\bm y \geq \bm y'$ and $F$ is a CDF, $F(\bm y) \geq F(\bm y')$. Hence, $F(\bm y) = F(\bm y')$. So, we have $P(f_1(\xxi) \leq y_1, ..., f_M(\xxi) \leq y_M) = P(f_1(\xxi) \leq y'_1, ..., f_M(\xxi) \leq y'_M)$. Suppose $y'_i = y_i$ for all $i\neq j$. Then, $P(f_1(\xxi) \leq y_1, y'_j < f_j(\xxi) \leq y_j, ..., f_M(\xxi) \leq y_M) = 0$. But the hyperrectangle bounded by $[-\infty,..., y'_j ,..., -\infty]$ and $\bm y$ has positive Lebesgue measure. Since the pdf of each of $f_1, ..., f_M$ is strictly positive, the cumulative density over the hyperrectangle is greater than zero, which is a contradiction. The argument is easily extended to the case when there exists $1\geq k \geq M$ indices $i_1, ..., i_k$ such that $y'_{i_k} < y_{i_k}$.
\end{proof}

\subsection{Extension to Black-Box Constraints Under Input Noise}
\label{appdx:const}


In this section, we consider the setting where in addition to the objective function $\bm f$ there is a vector-valued black-box function $\bm c(\bm x) \in \mathbb R^{V}$ specifying the outcome constraint $\bm c(\bm x) > \bm 0$ that is also subject to input noise $\bm \xi \sim P(\bm \xi)$. To handle black-box constraints under input noise, we weight the objectives $\bm f(\xxi)$ by a feasibility indicator $\mathbbm{1}[\bm c(\xxi) > \bm 0]$ that is 1 if all constraints are satisfied and 0 otherwise. We define \var{} and \mvar{} for feasibility-weighted objectives and extend the theoretical results from Section~\ref{subsec:theory_cheby}.

The proofs for the results in the constrained setting follow the proofs in Appendix~\ref{appdx:proofs_sub} with slight modifications. 1) Assumption~\ref{con:cont_cdf} does not hold for feasibility-weighted objectives. Hence, the implication is that a random sampled scalarization is only guaranteed to correspond to a point in the \weakmvar{} set, but importantly any point in the \mvar{} set does correspond to some scalarization and therefore can be recovered. 2) The proofs handle the special case where some of the constraints are not satisfied and the feasibility-weighted objectives are zero.

\begin{definition}
\label{def:const_var}
The value-at-risk of the feasibility-weighted objective for a given point $\bm x$ is:
\begin{equation*}
    \var{}_{ \alpha}\big(f(\xxi)\mathbbm{1}[\bm c(\xxi) > \bm 0] \big) = \sup \big\{z \in \mathbb R : P\big(f(\xxi)\mathbbm{1}[\bm c(\xxi) > \bm 0] \geq z\big) \geq \alpha\big\}.
\end{equation*}
\end{definition}
\begin{definition}
\label{def:const_mvar}
The $\mvar$ of the feasibility-weighted objectives $\bm f$ for a given point $\bm x$ is:
\begin{equation*}
\begin{aligned}
    \mvar{}_\alpha&\big(\bm f(\xxi), \bm c(\xxi) \big) = \\
    &\bigg\{ \bm z \in \mathbb R^M : P\big(\bm f(\xxi)\mathbbm{1}[\bm c(\xxi) > \bm 0] \geq \bm z\big) \geq \alpha, \nexists ~\bm z' \in \mathbb R^M, \bm z' \succ \bm z,  P\big(\bm f(\xxi)\mathbbm{1}[\bm c \succ \bm 0] \geq \bm z' \big) \geq \alpha \bigg\}.
\end{aligned}
\end{equation*}
Let $\mathcal{M}^{\bm{c}}_{\bm \xi, X}: = \bigcup_{\bm x \in X} \mvar{}_\alpha\big[\bm f(\xxi), \bm c(\xxi) \big]$. 
The \emph{global} $\mvar$ of the feasibility weighted objectives for a set of points $X$ is defined as
\begin{equation*}
    \mvar{}_{P(\bm \xi), \alpha}\big(\{\bm f(\xxi), \bm c(\xxi)\}_{\bm x \in X} \big) = \big\{\bm z \in \mathcal{M}^{\bm c}_{\bm \xi, X} : \nexists ~\bm z' \in \mathcal{M}^{\bm c}_{\bm \xi, X} ~s.t.~ \bm z' \succ \bm z \big\}.
\end{equation*}
\weakmvar{} of the feasibility-weighted objectives is defined in the same way, but only requires that its elements are weakly Pareto optimal
\end{definition}

\begin{lemma}\label{lemma:const_bool_equivalence}
Given a weight vector $\bm w \in \Delta_+^{M-1}$, $\bm y \in \mathbb R^M$, $\bm y_c \in \mathbb R^{M'}$ and $v \in \mathbb R$,
\begin{equation*}
    s[\bm y, \bm w]\mathbbm{1}[\bm y_c \succ \bm 0] \geq v \iff \bm y\mathbbm{1}[\bm y_c \succ \bm 0] \geq \frac{v}{\bm w}. 
\end{equation*}
\end{lemma}

\begin{proof}
    This follows directly from Definition~\ref{def:chebyshev}.
    \begin{align*}
        s[\bm y, \bm w]\mathbbm{1}[\bm y_c \succ \bm 0] \geq v &\iff \mathbbm{1}[\bm y_c \succ \bm 0]\min_i w_i y_i \geq v \\
        &\iff w_i y_i\mathbbm{1}[\bm y_c \succ \bm 0]  \geq v ~\forall i \\
        &\iff y_i\mathbbm{1}[\bm y_c \succ \bm 0] \geq \frac{v}{w_i} ~\forall i \\
        &\iff \bm y\mathbbm{1}[\bm y_c \succ \bm 0] \geq \frac{v}{\bm w}.
    \end{align*}
\end{proof}

\begin{theorem}[\var{} of Feasibility-Weighted Chebyshev scalarization $\Rightarrow$ Pareto Dominance]
\label{thm:const_cheby_dominance_cdf}
Given a weight vector $\bm w \in \Delta_+^{M-1}$, let 
$v = \var{}_\alpha\bigg(s[\bm f(\xxi), \bm w]\mathbbm{1}[\bm c(\xxi) > \bm 0]\bigg)$.
Then, 
\begin{equation*}
    P\bigg(\bm f(\xxi)\mathbbm{1}[\bm c(\xxi) > \bm 0] \geq \frac{v}{\bm w}\bigg) \geq \alpha.
\end{equation*}
\end{theorem}

\begin{proof}
From Definition~\ref{def:const_var}, we have that 
\begin{equation*}
    \var{}_\alpha\bigg(s[\bm f(\xxi), \bm w ]\mathbbm{1}[\bm c(\xxi) > \bm 0]\bigg)
    = \sup \bigg\{z \in \mathbb R : s[\bm f(\xxi), \bm w]\mathbbm{1}[\bm c(\xxi) > \bm 0] \geq z\big) \geq \alpha\bigg\}.
\end{equation*}
Hence, $$P(s[\bm f(\xxi), \bm w]\mathbbm{1}[\bm c(\xxi) > \bm 0] \geq v) \geq \alpha.$$ 
From Definition~\ref{def:chebyshev}, we have
$$P(\mathbbm{1}[\bm c(\xxi) > \bm 0]\min_i w_if^{(i)}(\xxi)\geq v) \geq \alpha.$$
By Lemma~\ref{lemma:const_bool_equivalence}, the following expressions are equivalent $$\mathbbm{1}[\bm c(\xxi) > \bm 0]\min_i w_if^{(i)}(\xxi) \geq v \iff \bm f(\xxi)\mathbbm{1}[\bm c(\xxi) > \bm 0] \geq \frac{v}{\bm w}.$$
Hence,
$P\big(\bm f(\xxi)\mathbbm{1}[\bm c(\xxi) > \bm 0] \geq \frac{v}{\bm w}\big) \geq \alpha.$
\end{proof}

\begin{lemma}
\label{lemma:zero_z}
Let $\bm z \in \mvar{}_\alpha\big[\bm f(\xxi), \bm c(\xxi)\big]$. If $\bm f(\bm x) > \bm 0$, then $\bm z \nsucc \bm 0$ if and only if $\bm z = \bm 0$.
\end{lemma}
\begin{proof}
The following shows that, since $\bm f(\bm x) > \bm 0$, $\bm z \nsucc \bm 0$ if and only if $P(\bm c(\xxi) > \bm 0) < \alpha$. Thus, $\bm z \nsucc \bm 0$ if and only if $\bm z = \bm 0$.

Since $\bm f(\bm x) > \bm 0$, we have that $z_i \geq 0 $ for all $i=1,..., M$ and there exists $j \in \{1,..., M\}$ such that $z_j = 0$. Note that since $\bm f(\bm x) > \bm 0$, the $i^\text{th}$ element $ f^{(i)}(\bm x)\mathbbm{1}[\bm c(\bm x) > \bm 0] = 0$ if and only if $\mathbbm{1}[\bm c(\bm x) > \bm 0] = 0$. Hence, either $\bm f(\bm x)\mathbbm{1}[\bm c(\bm x) > \bm 0] > \bm 0$ or $\bm f(\bm x)\mathbbm{1}[\bm c(\bm x) > \bm 0] = \bm 0$.

From Definition~\ref{def:const_mvar}, we have that $P\big(\bm f(\xxi)\mathbbm{1}[\bm c(\xxi) > \bm 0] \geq \bm z\big) \geq \alpha.$
Since $\bm f(\xxi)\mathbbm{1}[\bm c(\xxi) > \bm 0] > \bm 0$ or  $\bm f(\xxi)\mathbbm{1}[\bm c(\xxi) > \bm 0] = \bm 0$, $$P(\bm f(\xxi)\mathbbm{1}[\bm c(\xxi) > \bm 0] > \bm 0) + P(\bm f(\xxi)\mathbbm{1}[\bm c(\xxi) > \bm 0] = \bm 0) = 1.$$

Suppose $\bm z \nsucc \bm 0$. Since $\bm z$ is not dominated by any other point in the \mvar{} set, $P(\bm f(\xxi)\mathbbm{1}[\bm c(\xxi) > \bm 0] \succ \bm 0) < \alpha.$ Hence, $\bm z$ must be $\bm 0$. 
\end{proof}

\begin{lemma}[\var{} of Feasibility-Weighted Chebyshev scalarization $\Rightarrow$ Feasibility-Weighted \weakmvar{}]
\label{lemma:const_var_in_weak_mvar}
Let $v = \var{}_\alpha\big(s[\bm f(\xxi)\mathbbm{1}[\bm c(\xxi) > \bm 0], \bm w]\big)$. Then, $\frac{v}{\bm w} \in \weakmvar{}_\alpha\big(\bm f(\xxi), \bm c(\xxi)\big)$.
\end{lemma}
\begin{proof}
Without loss of generality, assume that the objectives $\bm f(\bm x) > \bm 0$.\footnotemark[\getrefnumber{footnote:pos_z}]
Let $\bm z = \frac{v}{\bm w}$. Suppose there exists $\bm z' \in \weakmvar{}_\alpha(\bm f(\xxi), \bm c(\xxi))$ such that $\bm z' > \bm z$. Since $\bm z' \in \weakmvar{}_\alpha(\bm f(\xxi), \bm c(\xxi))$, $$P(\bm f(\xxi)\mathbbm{1}[\bm c(\xxi) > \bm 0] \geq \bm z') \geq \alpha.$$
Note that
$\bm f(\xxi)\mathbbm{1}[\bm c(\xxi) > \bm 0] \geq \bm z'$ implies that $$\mathbbm{1}[\bm c(\xxi) > \bm 0]\min_i w_i f^{(i)}(\xxi) \geq \min_i w_iz_i.$$
Hence, $P(\bm f(\xxi) \mathbbm{1}[\bm c(\xxi) > \bm 0] \geq \bm z') \geq \alpha$ implies that $$P( s[\bm f(\xxi), \bm w] \mathbbm{1}[\bm c(\xxi) > \bm 0] \geq s[\bm  z', \bm w]) \geq \alpha.$$
Note that $s[\bm  z', \bm w] > s[\bm  z, \bm w] = v$. But by the Definition~\ref{def:const_var}, $v$ is the maximum value such that $$P( s[\bm f(\xxi), \bm w]\mathbbm{1}[\bm c(\xxi) > \bm 0] \geq v) \geq \alpha.$$
This is a contradiction.
\end{proof}

\begin{theorem}[Feasibility-Weighted \mvar{} $\Rightarrow$ \var{} of Feasibility-Weighted Chebyshev scalarization]
\label{thm:const_mvar_to_var}
For any $\bm z \in \mvar{}_\alpha\big[\bm f(\xxi), \bm c(\xxi) \big]$, there exists $\bm w \in \Delta^{M-1}_+$ such that $\bm z = \frac{v}{\bm w}$ where $v=\var{}_\alpha\big(s[\bm f(\xxi), \bm w]\mathbbm{1}[\bm c(\xxi) > \bm 0]\big).$
\end{theorem}
\begin{proof}
Without loss of generality, assume that the objectives $\bm f(\bm x) > \bm 0$.\footnotemark[\getrefnumber{footnote:pos_z}]

\textbf{Case 1}: $\bm z \nsucc \bm 0$. 
From Lemma~\ref{lemma:zero_z}, we have that $\bm z = \bm 0$. Note that since the \mvar{} set contains only non-dominated points and $\bm 0$ is a lower bound on $\bm f(\bm x)\mathbbm{1}[\bm c(\bm x) > \bm 0]$, $\mvar{}_\alpha\big[\bm f(\xxi), \bm c(\xxi) \big] = \{\bm 0\}.$ Let $v:= \var{}_\alpha\big(s[\bm f(\xxi), \bm w]\mathbbm{1}[\bm c(\xxi) > \bm 0]\big)$.

Suppose $v = 0$. Then, for any $\bm w \in \Delta_+^{M-1}, \frac{v}{\bm w} = \bm 0 \in \mvar{}_\alpha\big[\bm f(\xxi), \bm c(\xxi) \big]$.

Suppose $v > 0$. Then for any $\bm w \in \Delta_+^{M-1}, \frac{v}{\bm w} > \bm 0$. By Lemma~\ref{lemma:const_var_in_weak_mvar}, $\frac{v}{\bm w} \in \weakmvar{}_\alpha\big(\bm f(\xxi), \bm c(\xxi) \big)$. But $\mvar{}_\alpha\big[\bm f(\xxi), \bm c(\xxi) \big] = \{\bm 0\}$, and any point in the \mvar{} set is non-dominated. So $\frac{v}{\bm w} \nsucc \bm 0$. This is a contradiction.

Hence, $v = 0$. Therefore, Theorem~\ref{thm:const_mvar_to_var} holds when $\bm z \nsucc \bm 0$.

\textbf{Case 2}: $\bm z \succ \bm 0$.

Consider the vector $\frac{1}{\bm z}$.
If we divide $\frac{1}{\bm z}$ by its L1-norm, we obtain a vector $\bm w: = \frac{1}{\bm z}||\frac{1}{\bm z}||_1^{-1} \in \Delta^M_+$, where $||\cdot||_1$ denotes the L1-norm. Let $v = \frac{1}{||\frac{1}{\bm z}||_1} \in \mathbb R_+$. Then, $\bm z = \frac{v}{\bm w}$. Hence, all we need to show is that $v = \var{}_\alpha\big(s[\bm f(\xxi), \bm w]\mathbbm{1}[\bm c(\xxi) > \bm 0]\big).$
By definition, 
$$\var{}_\alpha\big(s[f(\xxi), \bm w]\mathbbm{1}[\bm c(\xxi) > \bm 0]\big)= \sup \big\{v''\in \mathbb R : P\big(s[f(\xxi), \bm w]\mathbbm{1}[\bm c(\xxi) > \bm 0] \geq v''\big) \geq \alpha \big\}.
$$
Let us define $v':=\var{}_\alpha\big(s[f(\xxi), \bm w]\mathbbm{1}[\bm c(\xxi) > \bm 0]\big)$. Suppose that  $v < v'$. By Lemma~\ref{lemma:const_var_in_weak_mvar}, 
$\frac{v'}{\bm w} \in \weakmvar{}_\alpha\big(\bm f(\xxi), \bm c(\xxi)\big).$ Since $v < v'$, 
$\bm z = \frac{v}{\bm w} \prec \frac{v'}{\bm w}.$
But $\bm z$ is in $\mvar{}_\alpha\big[\bm f(\xxi)\big]$. So by Definition~\ref{def:mvar}, $\bm z$ is not dominated by any other vector in $\weakmvar{}_\alpha\big(\bm f(\xxi), \bm c(\xxi)\big)$. This is a contradiction. Hence, 
$v = \var{}_\alpha\big(s[\bm f(\xxi), \bm w]\mathbbm{1}[\bm c(\xxi) > \bm 0]\big).$ Thus, Theorem~\ref{thm:const_mvar_to_var} holds when $\bm z \succ \bm 0$.
\end{proof}

\begin{corollary}
\label{cor:const_injection}
There is a injective function $g:\mvar{}_\alpha\big[\bm f(\xxi), \bm c(\xxi) \big] \rightarrow \Delta_+^{M-1}$ such that $g(\bm z) = \frac{1}{\bm z||\frac{1}{\bm z}||} = \bm w$ and $\bm z = \frac{v}{\bm w}$ where $v=\var{}_\alpha\big(s[\bm f(\xxi), \bm w]\big)=\frac{1}{||\frac{1}{\bm z}||}$. 
\end{corollary}

\begin{corollary}
\label{cor:const_consistent_optimizers} Suppose $\bm z$ is a point in the global \mvar{} set $\bm z \in \mvar{}_\alpha\big[\{\bm f(\xxi), \bm c(\xxi)\}_{\bm x \in \mathcal X}\big]$. Let $\mathcal X_{\bm{z}}^*$ be the set of designs such that for all $\bm x \in \mathcal X_{\bm{z}}^*$, $\bm z \in \mvar{}_\alpha\big[\bm f(\xxi), \bm c(\xxi)\big]$.
Then every $\bm x \in \mathcal X_{\bm{z}}^*$ is a maximizer of $\var{}_\alpha\big(s[\bm f(\xxi), \bm w]\mathbbm{1}[\bm c(\xxi) > \bm 0]\big)$.
\end{corollary}
When $\mvar{}_\alpha\big[\bm f(\xxi), \bm c(\xxi)\big] \neq \{\bm 0\}$, it follows directly from the injective mapping from $\bm z$ to $\bm w$ that $\var{}_\alpha\big(s[\bm f(\xxi), \bm w]\mathbbm{1}[\bm c(\xxi) > \bm 0]\big)$ is the same for all $\bm x \in \mathcal X_{\bm{z}}^*$. When $\mvar{}_\alpha\big[\bm f(\xxi), \bm c(\xxi)\big] = \{\bm 0\}$, then all $\bm x$ are infeasible designs and $\mathcal{X}^*_{\bm z} = \mathcal X$.

\section{\mars{} with Alternative Acquisition Functions}
\label{appdx:mars_acqfs}
In this section, we discuss using \mars{} with two alternative acquisition functions: Thompson Sampling (\textsc{TS}) and Upper Confidence Bound (\textsc{UCB}).

\subsection{\mars{} with Thompson Sampling}
As discussed in Section~\ref{sec:opt_mvar}, direct \mvar{} optimization with \qNEHVI{} requires evaluating the joint posterior over $n_{\bm{\xi}}(n+1)$ designs. The same is true when using \marsnei{}. Although low-rank Cholesky updates can significantly reduce the complexity \citep{Osborne2010BayesianGP}, further computational improvements can be obtained by using \textsc{TS} with random Fourier features (RFFs) \citep{Rahimi2008RFF}. However, RFFs are approximate GP samples and introduce approximation error (see Appendix~\ref{appdx:nehvi_rff} for further discussion).\footnote{Alternative approaches for efficient posterior sampling such as decoupled sampling \citep{wilson2020efficiently} could also be used.} We refer to this method as \marsts{}. \marsts{} naturally supports (i) parallel candidate generation by drawing a new posterior sample and new scalarization weights for each candidate; (ii) constraints, by evaluating the feasibility-weighted objectives under the posterior sample, and (iii) noisy observations.

\subsection{\mars{} with Upper Confidence Bound}
Another computationally efficient approach is to use Upper Confidence Bound (\textsc{UCB}), which does not require the expensive integration over $\bm f(\xxi)$ where $\bm x \in X_{1:n}$. We refer to this method as \marsucb{}. In what follows, we show how to extend the V-UCB algorithm of \citet{nguyen2021valueatrisk} to optimize the $\var{}$ of Chebyshev scalarizations. The result builds on the following lemma by \citet{Chowdhury2017IGP-UCB}, which holds under the assumption that the function $f^{(i)}(\cdot)$ belongs to a reproducing kernel Hilbert space (RKHS) $\mathcal{F}_{k_i}(B_i)$, whose RKHS norm is bounded by $\|f^{(i)}\|_{k_i} \leq B_i$, where 
$\bm f = [f^{(1)}, ..., f^{(M)}]$.
We use $\mu_n^{(i)}(x), \Sigma_n^{(i)}(x, x)$ to denote the posterior, conditional on observations up to iteration $n$, mean and variance of the GP surrogate corresponding to $i^{th}$ objective, and use $\sigma^2_i$ to denote the observation noise for the $i^{th}$ objective.


\begin{lemma} \label{lemma-chowdhury-ucb-lemma} \cite{Chowdhury2017IGP-UCB}. For $\delta \in (0, 1)$, $\zeta^{(i)}_{n + 1} = B_i  + \sigma_i^2 \sqrt{2 (\gamma_{n} + 1 + \log(1 / \delta))}$, the following holds for all $\bm x \in \mathcal{X}$ with probability $\geq 1 - \delta$:
\begin{equation} \label{eq-ucb-lemma-bounds}
    l^{(i)}_n(\bm x) \leq f^{(i)}(\bm x) \leq u^{(i)}_n(\bm x),
\end{equation}
where $l^{(i)}_n(\bm x) := \mu^{(i)}_n(\bm x) - \zeta^{(i)}_{n + 1} (\Sigma^{(i)}_n(\bm x, \bm x))^{1/2}$, $u^{(i)}_n(\bm x) := \mu^{(i)}_n(\bm x) + \zeta^{(i)}_{n + 1} (\Sigma^{(i)}_n(\bm x, \bm x))^{1/2}$, and $\gamma_n$ denotes the maximum information gain.
\end{lemma}
Assuming that each objective is modeled using an independent GP surrogate and considering all objectives jointly, we see that (\ref{eq-ucb-lemma-bounds}) holds jointly for all $i = 1, \ldots, m$ with probability at least $(1 - \delta') := (1 - \delta)^m$. Applying the Chebyshev scalarization,
\begin{align*}
    w_i l^{(i)}_n(\bm x) \leq w_i f^{(i)}(\bm x) \leq w_i u^{(i)}_n(\bm x), \forall i = 1, \ldots, M \\
    \min_i w_i l^{(i)}_n(\bm x) \leq \min_i w_i f^{(i)}(\bm x) \leq \min_i w_i u^{(i)}_n(\bm x) \\
    s[\bm l_n(\bm x), \bm w] \leq s[\bm f(\bm x), \bm w] \leq s[\bm u_n(\bm x), \bm w]
\end{align*}
holds with probability $\geq (1 - \delta')$ for all $\bm x \in \mathcal{X}$. Following Lemma 2 of \cite{nguyen2021valueatrisk}, we get that
\begin{equation*}
    \var{}_\alpha(s[\bm l_n(\xxi), \bm w]) \leq
    \var{}_\alpha(s[\bm f(\xxi), \bm w]) \leq
    \var{}_\alpha(s[\bm u_n(\xxi), \bm w])
\end{equation*}
holds with probability at least $(1 - \delta')$. Thus, the UCB policy for $\var{}$ of Chebyshev scalarization is defined as the policy that samples $\bm x_{n + 1} = \argmax_{\bm x} \var{}_\alpha(s[\bm u_n(\xxi), \bm w])$.

In practice, computing the $\zeta^{(i)}_{n + 1}$ given in Lemma \ref{lemma-chowdhury-ucb-lemma} is impractical, and typically leads to an acquisition function that is overly conservative. Thus, we follow \citet{nguyen2021valueatrisk} and use $\zeta^{(i)}_{n + 1} = 2 \log(n^2 \pi^2 / 0.6)$ in the experiments.

\marsucb{} can be extended to support parallel candidate generation by sampling a new scalarization for each candidate and noisy observations. The \textsc{UCB} policy derived above does not hold for feasibility-weighted objectives because Lemma \ref{lemma-chowdhury-ucb-lemma} requires that $f^{(i)}(\cdot)$ belongs to a RKHS, and this is not the case when $f^{(i)}(\cdot)$ is weighted by a feasibility indicator because it is no longer continuous \citep{Freitas2012ExponentialRB}.

\section{Gradient-based Acquisition Function Optimization}
\subsection{Approximate Gradients of $\var{}$}
\label{appdx:gradients-of-var}
One of the earliest and simplest-to-use gradient estimators for $\var$ was presented by \citet{Hong2009VaR}. Under mild regularity assumptions on the distribution of the random variable, they establish the consistency of the $\var{}$ gradient estimator, which can be seen as the sample-path gradient of the well-known estimator of $\var{}$. For this discussion, let $g(\cdot)$ be a deterministic function of its argument (e.g., a sample path of the GP), let us fix $\bm x$, and let $\bm \xi \sim P(\bm \xi)$ be a continuous random variable. Let $\bm \xi_1, \ldots, \bm \xi_k$ denote i.i.d. samples from $P(\bm \xi)$. Define the following ordering of the samples, where the subscript $(\cdot)$ denote the order statistic:
\begin{equation*}
    g(\xxi_{(1)}) \leq g(\xxi_{(2)}) \leq \ldots \leq g(\xxi_{(k)}).
\end{equation*}
The $\var{}$ at risk level $\alpha$ can be estimated by 
\begin{equation*}
    \var_{\bm \xi \sim P(\bm \xi)}(g(\xxi)) \approx g(\xxi_{(\lfloor (1 - \alpha) k \rfloor)}),
\end{equation*}
where $\lfloor \cdot \rfloor$ denotes the largest integer less than or equal to $\cdot$. It is well known (cf. \citet{Serfling1980}) that this estimator is consistent as $k \rightarrow \infty$. \citet{Hong2009VaR} extend this result to show that the corresponding gradient estimator
\begin{equation*}
    \nabla_{\bm x} \var_{\bm \xi \sim P(\bm \xi)}(g(\xxi)) \approx \nabla_{\bm x} g(\xxi_{(\lfloor (1 - \alpha) k \rfloor)})
\end{equation*}
is an asymptotically (as $k \rightarrow \infty$) unbiased estimator of the gradient of $\var{}$. The estimator is also consistent as long as $\nabla_{\bm x} g(\xxi_{(\lfloor (1 - \alpha) k \rfloor)})$ is not a function of $\xi$, otherwise, averaging of multiple sample gradients is required to obtain a consistent estimator of the gradient of $\var{}$.

In addition to the 
sample-path gradient estimator discussed above, there are other estimators of gradients of $\var{}$ that are based on, e.g., the likelihood ratio gradient estimation or on the kernel density estimators. A detailed discussion of these
can be found in \citet{Hong2014VaR-CVaR-Rev}.

\subsection{Approximate Gradients of $\mvar{}$} \label{appdx:gradients-of-mvar}
Differentiability of $\mvar{}$, more precisely the differentiability of the elements of the \mvar{} set, is a subject that has not been explored in the literature. Since the computation of the $\mvar{}$ set corresponding to a set of posterior samples is expensive enough to be the bottleneck during acquisition function optimization, it is highly desirable to avoid the finite-difference gradient estimation, which requires multiple evaluations of the objective and is in general less efficient than the sample-path gradients. Instead, it is preferable to establish a direct connection between the $\mvar{}$ set and the gradients of the samples on which the $\mvar{}$ is computed. The method we discuss below is inspired by the gradients of $\var{}$, which correspond to the gradients of the sample that is equal to $\var{}$.

The correspondence between $\nabla_{\bm x} \var_{\bm \xi \sim P(\bm \xi)}(g(\xxi))$ and $\nabla_{\bm x} g(\xxi_{(\lfloor (1 - \alpha) k \rfloor)})$ follows from the observation that, since $g(\cdot)$ is a continuous function, shifting $\bm x$ by a sufficiently small $\bm \epsilon$ should not change the ordering of $\bm \xi$'s. We should still have $\var_{\bm \xi \sim P(\bm \xi)}(g(\bm x + \bm \epsilon + \bm \xi)) \approx g(\bm x + \bm \epsilon + \bm \xi_{(\lfloor (1 - \alpha) k \rfloor)})$ with the same ordering, as long as $g(\xxi_{(i)}) \neq g(\xxi_{(j)})$ for $i \neq j$. 
The same idea extends to the $\mvar{}$. Using a finite set of samples to approximate the $\mvar{}$ set, with $\bm m (\xxi)$ denoting an arbitrary element of the $\mvar{}$ set, we have that $m^{(j)} (\xxi) = f^{(j)} (\xxi_{i^{(j)}})$ for some $i^{(j)} \in \{1, \ldots, k\}$, where
$\bm f = [f^{(1)}, ..., f^{(M)}]$ and the $i^{(j)}$ is dependent on the outcome $j$ and the particular element of the $\mvar{}$ set. This can be interpreted as saying that the elements of the $\mvar$ set are constructed by piecing together outcomes from the samples of the random variable.

Similar to what was discussed for $\var{}$, if we perturb $\bm x$ by a small $\bm \epsilon$, under the assumption that $f^{(j)}(\xxi_i) \neq f^{(j)}(\xxi_k)$ for $i \neq k$, we should get that $m^{(j)} (\bm x + \bm \epsilon + \bm \xi) = f^{(j)} (\bm x + \bm \epsilon + \bm \xi_{i^{(j)}})$ with the same $i^{(j)}$ as before the perturbation. 
This, in essence, says that we can calculate the gradients of the elements of the $\mvar$ set as $\nabla_{\bm x} m^{(j)} (\xxi) = \nabla_{\bm x} f^{(j)} (\xxi_{i^{(j)}})$. Putting all outcomes together, we get
\begin{equation*}
    \nabla_{\bm x} \bm m (\xxi) = [\nabla_{\bm x} f^{(1)} (\xxi_{i^{(1)}}), \ldots, \nabla_{\bm x} f^{(M)} (\xxi_{i^{(M)}})].
\end{equation*}

A theoretical consistency analysis of these \mvar{} gradient estimators is beyond the scope of this paper. However, we observe that they do work well in practice, enabling efficient optimization of \mvarnehvirff{} (see Appendix~\ref{appdx:direct-mvar-opt}\&\ref{appdx:additional_experiments}).


\section{Direct \mvar{} Optimization using \qNEHVI{}} \label{appdx:direct-mvar-opt}
In this section, we discuss direct optimization of \mvar{} using \NEHVI, highlight the computational challenges that come with this approach, and introduce an approximation that mitigates some of these challenges for some problems where the $\mvar{}$ set for a design is relatively small (e.g. where the number of objectives is small and $\alpha$ is large). However, we find that these approaches are typically infeasible when $M\geq3$ due to GPU memory limits (see for example Table~\ref{table:additional-runtime}).

\subsection{Direct Optimization of $\mvar{}$ with $\NEHVI$} \label{appdx:mvar-nehvi}

As described in Section~\ref{sec:opt_mvar}, the extension of \qNEHVI{} to optimize $\mvar{}$ is conceptually simple. 
\qNEHVI{} selects the next point to evaluate by maximizing the expected \HVI{} under the GP posterior.
We replace the standard \HVI{} of a new point with respect to the PF with the joint \HVI{} of the $\mvar{}$ set of a new point with respect to the $\mvar{}$ set over the previously evaluated designs:
\begin{equation*}
        \alpha_{\mvarnehvi}(\bm x) = \mathbb E_{\bm f\sim P(\bm f | \mathcal D)}
        \big[\HVI{}(\mvar{}_\alpha[\bm f(\xxi)] ~|~ \mvar{}_\alpha[\{\bm f(\xpxi)\}_{\bm x' \in X_{1:n}}]\big].
\end{equation*}

However, there are several computational issues that make this approach prohibitively expensive, except for when the objective evaluation takes multiple hours or days. In our experiments, we observe long runtimes, even when using very reasonable parameter values of $n_{\bm{\xi}} = 32$ and $\alpha = 0.9$.
We discuss the factors contributing to this in the following subsection. Note that all the issues discussed are compounded by the fact that even when using a gradient-based approach to optimize the acquisition function, the acquisition function needs to be evaluated many times.

\subsection{Complexity and Challenges}
\label{appdx:complexity}
There are three primary computational bottlenecks, corresponding to three stages of computing $\alpha_{\mvarnehvi}(\bm x)$. We discuss each stage and their complexity below.


\begin{figure*}[ht]
    \centering
    \includegraphics[width=0.3\linewidth]{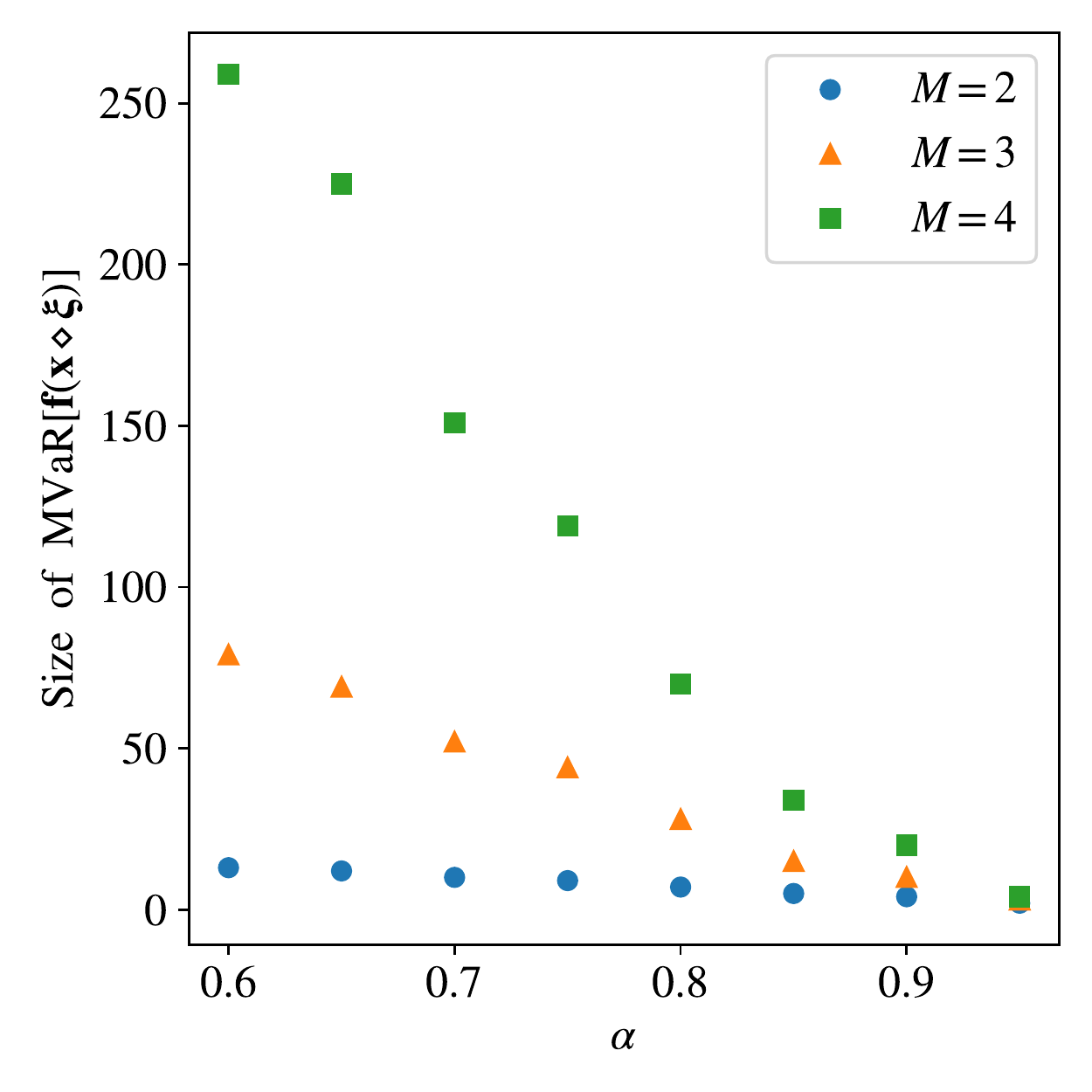}
    \caption{The maximum size of $\mvar{}_\alpha[\bm f(\xxi)]$ across the design space $\bm x \in \mathcal X$, which is an estimator of the maximum \mvar{} set size that will encountered during numerical optimization, for different $\alpha$ with $n_{\bm\xi}=32$ on the GMM problem.  The size of the \mvar{} set significantly increases as $\alpha$ decreases and as $M$ increases.}
    \label{fig:mvar_set_size}
\end{figure*}

\begin{enumerate}
    \item \textbf{Posterior Sampling}: Computing $\alpha_{\mvarnehvi}(\bm x)$ requires drawing joint posterior samples at the baseline points (points that are already evaluated) and the current candidate(s) $\bm x$ under all $n_{\bm{\xi}}$ perturbations.
    For all methods using GPs, posterior sampling at $n$ points and $n_{\bm{\xi}}$ perturbations scales as $O(n_{\bm{\xi}}^3n^3M)$. Hence, using GPs is only feasible for modest $n_{\bm{\xi}}$.
    
    \item \textbf{Computing \mvar{}}: The next stage is computing the \mvar{} corresponding to each $\bm x$ from posterior samples of $\bm f(\xxi)$. 
    Computing \mvar{} involves computing the distribution function of $\bm f(\xxi)$, \footnote{The distribution of $\bm f(\xxi)$ is unknown and the GP posterior over $\bm f(\xxi)$ is generally intractable. Hence, the distribution of $\xxi$ is often approximated with a finite set of MC samples \citep{Cakmak2020borisk} because the GP posterior can be evaluated analytically over a finite set of points.} which has a time and space complexity of $\mathcal{O}((1 - \alpha)^M n_{\bm{\xi}}^{M+1} M)$.
    The \mvar{} has to be computed for each $\bm x$ and each posterior sample, further inflating the computational effort required.
    The size of the resulting \mvar{} set is $\mathcal{O}((1 - \alpha)^M n_{\bm{\xi}}^M M)$. Figure~\ref{fig:mvar_set_size} empirically demonstrates that the size of the \mvar{} set significantly increases as $\alpha$ decreases, particularly for larger~$M$.
    
    \item \textbf{Computing joint hypervolume improvement}: Given the samples of \mvar{} corresponding to the baseline points and the candidate(s), the final step of \mvarnehvi{} is to compute the joint \HVI{} of the \mvar{} set of the candidates over the global \mvar{} set corresponding to the baseline points. 
    To our knowledge, the only existing differentiable approach for joint \HVI{} computation relies on the inclusion-exclusion principle (\iep{}, \citet{daulton2020ehvi}). 
    The time and space complexity of computing the joint \HVI{} of a set of~$q'$ points using the \iep{}
    is exponential with respect to~$q'$.
    To compute the \HVI{} of the \mvar{} for a set of $q$ candidates, we must compute the joint \HVI{} of a set of size $q' = q|\mvar{}_\alpha[\{\bm f(\bm x_i + \bm \xi)\}_{i=1}^q)]|$.
    Since the size of the \mvar{} set of a single candidate scales as $\mathcal{O}((1 - \alpha)^M n_{\bm{\xi}}^MM)$, using \iep{} quickly becomes infeasible, except for very moderate $M, n_{\bm{\xi}},$ and $\alpha$.
    As shown in Figure~\ref{fig:mvar_set_size}, even for $q=1, n_{\bm{\xi}}=32,$ and $M=3$, the size of the \mvar{} set can be quite large for smaller values of $\alpha$, which precludes the use of \iep{}.
    Although \iep{} is necessary to make the joint hypervolume improvement computation differentiable, there are non-differentiable approaches (e.g. \citet{LACOUR2017347}) that could be used instead to compute the joint hypervolume improvement. However, optimizing $\alpha_{\mvarnehvi}$ without gradients would be very slow given that \citet{daulton2020ehvi} showed that simply optimizing analytic EHVI (without \mvar{}) with CMA-ES \citep{cmaes} or L-BFGS-B \citep{Byrd1995ALM} with approximate gradients estimated via finite differences is over an order of magnitude slower than when using exact gradients.
\end{enumerate}

A final challenge in using \mvarnehvi{} is that the calculation of \mvar{} is not differentiable, and there are no known theoretically-grounded gradient estimators of \mvar{}. Therefore, we use the heuristic approach described in Appendix~\ref{appdx:gradients-of-mvar} for estimating the gradients of \mvar{}.

\subsection{Approximating $\qNEHVI$ with RFF Draws}
\label{appdx:nehvi_rff}
Random Fourier Features (RFF, \citet{Rahimi2008RFF}) offer an inexpensive and differentiable approximation of GP sample paths. \citet{daulton2021parallel} propose to combine a single RFF draw with $\NEHVI$ to obtain a cheap approximation, \TSHVI{}. They find that \TSHVI{} is competitive with $\qNEHVI$ in small dimensional search spaces, though its performance degrades as the dimensionality increases.

We follow their approach and extend \TSHVI{} to optimize $\mvar$, and name this method \mvarnehvirff{}. Using RFFs significantly reduces the computational cost of the evaluating the acquisition function because it avoids computationally expensive exact posterior sampling. In addition, since we use a single RFF draw, the $\mvar{}$ set and its \HVI{} only have to be computed once per acquisition function evaluation rather than for each posterior sample. In the end, for small problem instances, \mvarnehvirff{} ends up with a per-iteration runtime measured in seconds, which is an immense reduction from the time it takes for \mvarnehvi{} using an exact GP model. 

Note that \mvarnehvirff{} still requires the use of \iep{} for differentiable \HVI{} computations, making this approach infeasible in many settings as discussed above. 
In addition, the performance of RFFs based acquisition functions are known to degrade when the underlying function is difficult to model and variance starvation is a known issue \citep{wang2018batched, wilson2020efficiently, mutny2018, calandriello2019gaussian}. Thus, for an acquisition function that works well in all settings, we recommend using \mars{}.

\section{Pruning for Efficient Joint Posterior Sampling}
\label{appdx:pruning}
\nei{} and \qNEHVI{} both require sampling from the joint posterior over function values at the new design $\bm x$ and previously evaluated designs $X_{1:n}$: $P(\{\bm f(\bm x)\}_{\bm x \in X_{1:n} \cup \{\bm x_1, \ldots, \bm x_q\}} | \mathcal D)$. To reduce the cost of posterior sampling, we prune $X_{1:n}$ to only include the subset of points $X_\text{pruned} \subseteq X_{1:n}$ that have nonzero probability of being optimal. We estimate the probability of being optimal using MC estimation with $N_\text{prune}$ samples from the joint posterior. For \nei{}, optimality is with respect to a scalar objective and often means $|X_\text{pruned}| << n$. For \qNEHVI{}, any design that has nonzero probability of being Pareto optimal is retained; typically, this results in a much larger $X_\text{pruned}$ than we using \nei{} with a scalar objective. The typically larger size of $X_\text{pruned}$ under \qNEHVI{}-based methods has a significant effect when using \mvar{}\textsc{-}\qNEHVI{}, where sampling from the joint posterior scales as $\mathcal O(Mnn_{\bm \xi})$ and has a very significant effect on runtime. Pruning strategies that leverage techniques from pre-screening and population selection in EAs may further improve computational efficiency.

\section{Optimization of Multi-Objective Expectation Objectives}
\label{appdx:expectation-optimization}
As noted in the main text, the optimization of expectation of objectives can be achieved via rather straightforward extensions of the existing multi-objective acquisition functions. Here, we discuss the main idea, and show how to extend \qNParego{} and \qNEHVI{} \citep{ daulton2021parallel}.

\subsection{Optimization Expectation Objectives with \qNParego{}}
The acquisition function $\textsc{ParEGO}$ is an extension of the well known Expected Improvement acquisition function to the multi-objective setting via augmented Chebyshev scalarizations. \qNParego{} is an MC-based variant that uses composite objectives with the \nei{} acquisition function \citep{daulton2021parallel}.  Given a weight vector $\bm w \in \Delta_+^{M-1}$, it selects the next point to evaluate as follows:
\begin{equation} \label{eq-qparego}
    \bm x_{n+1} = \argmax_{\bm x \in \mathcal{X}} \mathbb{E}_{\bm f \sim P(\bm f | \mathcal D)} \left[ s_a[\bm f(\bm x), \bm w] - \max_{\bm x' \in X_{1:n}} s_a[\bm f(\bm x'), \bm w] \right]_+,
\end{equation}
where $X_{1:n}$ denotes the points evaluated so far, and $[\cdot]_+$ denotes $\max(\cdot, 0)$, $s_a[\bm y, \bm w] = \min w_iy_i + \beta \sum_i w_iy_i$, and $\beta$ is a small positive constant. The expectation in \eqref{eq-qparego} is not available in closed form, and is typically replaced by a (Q)MC approximation obtained by drawing samples of $\{\bm f(\bm x')\}_{\bm x' \in X_{1:n} \cup \{\bm x\}}$ from the joint GP posterior. A batch of $q$ candidates can be selected in a sequential greedy fashion where each point is selected using a different scalarization weight vector and the improvement from the batch of $q$ points replaces the improvement from a single point in \eqref{eq-qparego}.

To extend $\qNParego{}$ to the expectation objectives, we replace each occurrence of $s_a[\bm f(\bm x), \bm w]$ in (\ref{eq-qparego}) with $s_a \left[\mathbb{E}_{\bm \xi \sim P(\bm \xi)}[\bm f(\xxi)], \bm w \right]$. For implementation, we follow the same MC idea, and draw samples from the joint posterior of $\{\bm f(\xxi)\}_{\bm x \in X_{1:n} \cup \{\bm x_1, \ldots, \bm x_q\}, \bm \xi \in \Xi}$ where $\Xi$ is a set of $n_{\bm \xi}$ input noise samples, and approximate (\ref{eq-qparego}) using these samples. To improve the computational efficiency, one can also calculate the posterior distribution of $\{\mathbb{E}_{\bm \xi \sim P(\bm \xi)}[\bm f(\xxi)] \}_{\bm x \in X_{1:n} \cup \{\bm x_1, \ldots, \bm x_q\}}$ from the posterior of $[\bm f(\xxi)]_{\bm x \in X_{1:n} \cup \{\bm x_1, \ldots, \bm x_q\}, \bm \xi \in \Xi}$ via a simple matrix-matrix product, and use that to draw the posterior samples. This avoids inverting $m$ $(n + q) n_{\bm{\xi}} \times (n + q) |\Xi|$ matrices, and reduces the cost of posterior sampling from $\mathcal{O} (m (n + q)^3 n_{\bm{\xi}}^3)$ to $\mathcal{O} (m (n + q)^3)$. However, this computational technique cannot be used if there are black-box constraints. We refer to this method as \textsc{Exp-}\qNParego{}.

\subsection{Optimization Expectation Objectives with \qNEHVI{}}
The other acquisition function we consider is the \qNEHVI{}:
\begin{equation} \label{eq-qnehvi}
    \bm x_{n+1} = \argmax_{\bm x \in \mathcal{X}} \mathbb{E}_{\bm f \sim P(\bm f | \mathcal D)} \left[ \HVI{}(\bm f(\bm x) |  \mathcal P_n ) \right],
\end{equation}
where $\mathcal P_n$ is the PF over $\bm f(\bm x)_{\bm x' \in X_{1:n}}$.  The extension of $\qNEHVI$ to the expectation objectives follows a similar path to that of \qNParego{}. We replace the hypervolume improvement of $\bm f(\bm x)$ in (\ref{eq-qnehvi}) with the hypervolume improvement of $\mathbb{E}_{\bm \xi \sim P(\bm \xi)}[\bm f(\bm x)]$ with respect to the PF over $\{\mathbb{E}_{\bm \xi \sim P(\bm \xi)}[\bm f(\bm x')]\}_{\bm x' \in X_{1:n}}$. To do so, we replace the posterior samples of $\bm f(\bm x)$ and $\{\bm f(\bm x')]\}_{\bm x' \in X_{1:n}}$ that are used in \qNEHVI{} calculations with the posterior samples of the expectation over $P(\bm \xi)$, which can be obtained in the same manner described above for \qNParego{}. However, \qNEHVI{} with expectation objective is often prohibitively slow because typically relatively few points from $X_{1:n}$ can be pruned as discussed in Appendix~\ref{appdx:pruning}. Hence, we only evaluate a single sample approximation using RFFs, analogous to the RFF approximation of \mvarnehvi{}, which we refer to as \expnehvirff{}.

\subsection{Challenges of Using Expectation with Feasibility-Weighted Objectives}
Independently computing the expectation of the objectives and the feasibility and taking the product of the expectations, would ignore the fact that the objective functions and constraint functions are evaluated on the same perturbed designs. To account for the perturbed inputs jointly across in the objectives and constraints, we use feasibility weighted objectives. Feasibility weighting requires penalizing designs that are infeasible such that the feasibility-weighted objectives for an infeasible design are worse than the objectives for any feasible design.

Feasibility weighting can make the expectation sensitive to the range of the objectives.
When evaluating a solution near the border of the feasible domain, we end up with a subset of the perturbed solutions evaluating to zero due to infeasibility and others evaluating to their respective objective values. To see how this can affect the performance, consider the following examples. Suppose that half of the perturbed solutions are infeasible and the objective values are bounded in $[0, 1]$. In this case, the feasibility weighted objective take values in $[0, 0.5]$, where it will be inferior to some other solutions due to the potential for it to be infeasible. Now, suppose that the objective values are bounded in $[100, 101]$. The infeasibility in this case will bring the feasibility weighted objective to the range of $[50, 50.5]$, which is strictly worse than any solution that is more feasible, even if by only a small fraction. If we instead set the infeasible solutions to $100$ rather than zero, this would lead to the feasibility weighted objective value to $[100, 100.5]$.  Note that setting the infeasible solutions to $100$ is equivalent to normalizing the objectives to $[0, 1]$ before applying the feasibility and using zero for the infeasible objectives, which in theory should have no effect in the optimization performance. In practice, we typically do not know precise bounds on the objectives, and instead standardize / normalize the objectives during optimization using bounds derived on the go. 

The example above highlights the effect of the range of each objective. Often, an infeasibility cost  $\lambda$ is used to penalize for infeasible designs to ensure that infeasible points are worse than any feasible point (for example if the objectives can take negative values) by setting the feasibility-weighted objectives to $(\bm f(\bm x) + \lambda)\mathbbm{1}[\bm c > \bm 0] - \lambda$ for some $\lambda \geq 0$. However, the feasibility weighted expectation is typically sensitive to the infeasibility cost, and feasibility weighted objectives give higher value to conservative solutions if the Pareto front lies near the border of the feasible domain. This makes it difficult to determine apriori how conservatively expectation methods will act when using feasibility weighted objectives. In contrast, \mvar{} avoids this issue by providing high probability guarantees on the value of the feasibility-weighted objectives under input noise. For any that design is feasible with probability $\alpha$, the infeasibility cost is in the tail of the multivariate CDF and has no effect on the elements of the \mvar{} set.

\section{Experiment Details}
\label{appdx:experiments}

\subsection{Method Details}
\label{appdx:method_details}
We evaluate the following BO methods:

\textbf{Methods that optimize the nominal objectives} (see \citet{daulton2021parallel}) for details): \qNParego{}, \qNEHVI{}, and \textsc{NEHVI-RFF} (referred to as \TSHVI{} in \citet{daulton2021parallel}), which approximates the expectation in \qNEHVI{} with a single approximate GP sample using RFFs.

\textbf{Methods that optimize the expectation objectives} (Appendix~\ref{appdx:expectation-optimization}): \textsc{Exp-}\qNParego{}, and \expnehvirff{}.

\textbf{Methods that optimize \mvar{}}: \marsnei{} (Section~\ref{subsec:random_scalar}), \marsts{} (Appendix~\ref{appdx:mars_acqfs}), \marsucb{} (Appendix~\ref{appdx:mars_acqfs}), \mvarnehvi{} (Appendix~\ref{appdx:direct-mvar-opt}), and \mvarnehvirff{} (Appendix~\ref{appdx:direct-mvar-opt}).

We implemented all methods using the BoTorch library \citep{balandat2020botorch} (except for NSGA-II), leveraging the existing implementations of \nei{} and \qNEHVI{} available at \url{https://github.com/pytorch/botorch}. We used the implementation of NSGA-II in the PyMOO library \citep{pymoo}, which is available at \url{https://github.com/anyoptimization/pymoo}.

For all model-based methods, we model each objective and constraint with an independent GP with a Mat\'{e}rn-$\frac{5}{2}$ ARD kernel \citep{Rasmussen2004}.\footnote{All methods except \marsucb{} support using multi-task GPs that model for correlations between objectives  \citep{NIPS2007_66368270}.} For methods that use scalarizations, we use composite objectives \citep{astudillo2019composite}. We use maximum a posteriori estimates of the GP hyperparameters using the default priors in BoTorch. 
For all MC-based acquisition functions, we use $N_{MC} = 256$ QMC samples from the GP posterior.
We use sample-average approximation \citep{balandat2020botorch} by using fixed Quasi-MC samples from $P(\bm\xi)$ (for robust methods) and fixed Quasi-MC base samples for all methods to approximate the expectation over the GP posterior.\footnote{All fixed samples are re-sampled once per BO iteration.}\footnote{For heteroskedastic input noise processes, we fix a set base samples and use the reparameterization trick \citep{kingma2013reparam} to sample from from the heteroskedastic input noise process using the fixed based samples---rather than directly fixing the input noise samples as we do in the case of homoskedastic noise.} 
This results in an approximation of the acquisition function that is a deterministic function of the input $\bm x$. For RFF-based methods, the approximate GP sample (using 512 random features) is also a deterministic function, which, coupled with fixed samples from $P(\bm\xi)$, results in a deterministic approximation of the acquisition functions. 
The deterministic approximations of the acquisition functions enable the use of quasi-Newton methods for optimization.
We optimize all acquisition functions using multi-start optimization with L-BFGS-B \citep{Zhu1997LBFGS}.

For \mars{} and other \mvar{}-based methods, we use the known \mvar{} reference point, which would typically be supplied by the decision maker. For \qNEHVI{} and \expnehvirff{},  we use the heuristic from \citet{daulton2020ehvi} to adaptively infer the the reference point during the optimization (the \mvar{} reference point is not suitable for the nominal and expectation objectives).

For methods involving scalarizations, the objectives are normalized before applying the scalarizations. For \mars{} methods, the reference point is used as the lower bound and the ideal point (i.e. the component-wise maximum of each objective, \citet{ishibuchi18}) across the \mvar{} set over the previously evaluated designs (estimated using the posterior mean) is used as the upper bound for normalization. For \qNParego{}, we use the ideal and nadir points (i.e. the component-wise minimum objective values across the PF, \citet{ishibuchi18}) across the PF over the previously evaluated designs. Similarly, for \textsc{Exp-}\qNParego{}, we use the ideal and nadir points over the PF expectation objectives (estimated using the posterior mean) over the previously evaluated designs.

For feasibility weighting the objectives, we use a sigmoid function as a differentiable approximation of the indicator function as in \citet{balandat2020botorch}. 
The infeasibility cost set to be the minimum posterior mean minus six standard deviations.

For methods that use \nei{} and \qNEHVI{} with exact posterior sampling, we prune the previously evaluated designs using $N_\text{prune}=2048$ samples to estimate the probability that a previously evaluated design is optimal. Additionally, we cache the Cholesky decomposition of the posterior covariance matrix over $\{\bm f(\xpxi)\}_{\bm x' \in X_\text{pruned}}$ and use low-rank updates to draw joint samples over $\{\bm f(\xpxi)\}_{\bm x' \in X_\text{pruned}\cup \{\bm x\}}$ \citep{Osborne2010BayesianGP}.
    

For batch (or asynchronous) candidate generation, we use a sequential greedy approach \citep{wilson2018maxbo}, where one new candidate is optimized at time and the joint acquisition value of all candidates $\bm x_1,..., \bm x_i$ is optimize to select $\bm x_i$. For methods relying on scalarizations, a new scalarization is sampled for each new candidate.

For NSGA-II, we used the same initial sobol starting points as for the other methods. We used a population size of 10 and adjusted the number of iterations for NSGA-II according to the evaluation budget. The PyMOO default configuration was used for all other settings. In addition to the objectives, observations of the constraints are provided to NSGA-II.

\subsection{Problem Details} \label{appdx:problem_details}
In this section, we provide descriptions of the test problems. The reference points used for all problems are provided in Table~\ref{table:ref_points}. For each problem, we set the reference point to be slightly worse than the nadir point (using the heuristic from \citet{ishibuchi2011}) of the \mvar{} set evaluated over a large grid of design points. For the Penicillin problem, we set the reference point to exclude the region of the objective space with low values of the time objective (following \cite{liang2021scalable}) as those objective trade-offs are less appealing to decision makers due to providing negligible Penicillin yield.

\textbf{Toy Problem} ($d=1,~M=2,~ \alpha=0.9$): 
This is the toy problem that was used to highlight the concepts in Figures~\ref{fig:1d_toy}\&\ref{fig:bijection}. The noise model is given by $P(\bm \xi) = \mathcal{N}(\mu = \bm 0, \Sigma = 0.1 I_2)$. The first objective is a mixture linear-sinusoidal function, and the second objective is modified from the well-known Levy test function. The exact expressions are given as follows. The function is evaluated on $x \in [0, 0.7]$.
\begin{align*}
    f^{(1)}(x) &= 30 - 30 * \left( p_1(x) p_4(x) + p_2(x) (1 - p_4(x)) + p_3(x) \right) \\
    p_1(x) &= 2.4 - 10 x - 0.1 x^2 \\
    p_2(x) &= 2 x - 0.1 x^2 \\
    p_3(x) &= (x - 0.5)^2 + 0.1 \sin(30 x) \\
    p_4(x) &= 1 / \left(1 + \exp \left( (x - 0.2) / 0.005 \right) \right) \\
    f^{(2)}(x) &= p_5\left((x * 0.95 + 0.03) * 20 - 10 \right) \\
    p_5(x) &= p_6\left(1 + (x - 1) / 4\right) - 0.75 * x^2 + 9.0955 \\
    p_6(x) &= \left(\sin(\pi * x)\right)^2 + (x - 1)^2 (1 + 10 \left(\sin(\pi * x)\right)^2) 
\end{align*}

\textbf{GMM} ($d=2,~M\in\{2,3,4\},~ \alpha\in\{0.7, 0.8, 0.9\}$):
In addition to the version presented in the main text, we consider several variations of the GMM problem using different number of objectives, different noise models, and different risk levels to analyze the effects of these factors on the optimization performance of the algorithms.
For all GMM problems considered, each objective is a mixture of the probability density function of three Gaussian distributions, modified from the single objective version presented in \cite{nes}. We present the canonical formula of the objectives and the parameters corresponding to each objective below. In the formula, $\phi(\bm x; \mu, \Sigma)$ is used to denote the probability density function of the multivariate Gaussian distribution with mean $\mu$ and covariance matrix $\Sigma$. The search space for all GMM problems is $\mathcal X = [0,1]^2$.

The GMM problem used in the main text involves 2 objectives and uses $\alpha=0.9$. Additional experiments in Appendix~\ref{appdx:additional-problems} use additional independent GMMs to increase the number of objectives to $3$ and $4$, and evaluate performance with different settings of $\alpha \in \{0.7,0.8,0.9\}$. 
In all experiments with 3 and 4 objective GMM, we use additive noise, where $P(\bm \xi) = \mathcal{N}(\mu = \bm 0, \Sigma = 0.05 I_M)$. Many additional noise processes are discussed and evaluated in Appendix~\ref{appdx:gmm_noise_level}, using the same 2 objective GMM problem from the main text and $\alpha = 0.9$.

\begin{align*}
    f^{(i)}(\bm x) &= 2 \pi \sum_{j=1}^3 var^{(i)}_j cons^{(i)}_j \phi(\bm x; \mu = pos^{(i)}_j, \Sigma = var^{(i)}_j I_2) \\
    pos^{(i)}_j &= \left\{
    \begin{array}{cccc} 
        j=1 & j=2 & j=3 & \ \\
        {}[0.2, 0.2] & [0.8, 0.2] & [0.5, 0.7] & \text{if } i = 1 \\
        {}[0.07, 0.2] & [0.4, 0.8] & [0.85, 0.1] & \text{if } i = 2 \\
        {}[0.08, 0.21] & [0.45, 0.75] & [0.86, 0.1] & \text{if } i = 3 \\
        {}[0.09, 0.19] & [0.44, 0.72] & [0.89, 0.13] & \text{if } i = 4
   \end{array}
   \right. \\
    var^{(i)}_j &= \left\{
    \begin{array}{cccc} 
        j=1 & j=2 & j=3 & \ \\
        0.04 & 0.01 & 0.01 & \text{if } i = 1 \\
        0.04 & 0.01 & 0.0025 & \text{if } i = 2 \\
        0.04 & 0.01 & 0.0049 & \text{if } i = 3 \\
        0.0225 & 0.0049 & 0.0081 & \text{if } i = 4
   \end{array}
   \right. \\
    cons^{(i)}_j &= \left\{
    \begin{array}{cccc} 
        j=1 & j=2 & j=3 & \ \\
        0.5 & 0.7 & 0.7 & \text{if } i = 1 \\
        0.5 & 0.7 & 0.7 & \text{if } i = 2 \\
        0.5 & 0.7 & 0.9 & \text{if } i = 3 \\
        0.5 & 0.7 & 0.9 & \text{if } i = 4
   \end{array}
   \right.
\end{align*}

\textbf{Constrained Branin Currin}
We use the open source implementation available at \url{https://github.com/pytorch/botorch}. See \citet{daulton2020ehvi} for details.

\textbf{Disc Brake}
We use the open source implementation available at \url{https://github.com/ryojitanabe/reproblems}. See \citet{tanabe2020} for details.

\textbf{Penicillin Manufacturing Problem} 
We use the open-source implementation available at \url{https://github.com/HarryQL/TuRBO-Penicillin}. See \citet{liang2021scalable} for details. We adapt the problem by adding independent zero-mean Gaussian input noise to each parameter. The standard deviation of the input noise distribution for each parameter is listed in Table~\ref{table:penicillin_noise}.
\FloatBarrier
\begin{table*}[ht]
	\caption{Standard deviation for independent zero-mean Gaussian input noise for each parameter in the Penicillin Problem (reported as a percentage of the range of each parameter).}
	\label{table:penicillin_noise} 
	\centering
	\small{\tabcolsep=0.11cm
    \begin{tabular}{lc}
        \toprule
    	Parameter            & Noise Level          \\
    	\midrule
    	Culture Volume & 3\% \\ 
    	Biomass Concentration & 3\%\\
    	Temperature & 0.5\% \\
    	Glucose Concentration & 2\%\\
    	Substrate Feed Rate & 1\%\\
    	Substrate Feed Concentration &1\%\\
    	$\text{H}^+$ Concentration &1\%\\
        \bottomrule
    \end{tabular}
    }
\end{table*}
\FloatBarrier

\begin{table*}[ht]
	\caption{Reference points for negative versions (i.e. multiplying the objectives by -1 to make the goal maximization of all objectives) of all problems (except the GMM and Toy problems, which are designed for maximization).}
	\label{table:ref_points} 
	\centering
	\small{\tabcolsep=0.11cm
    \begin{tabular}{lc}
        \toprule
    	Problem             & Reference Point          \\
    	\midrule
    	Toy Problem & [-14.1951,  -3.1887]\\
    	Disc Brake & [-5.89, -3.27]\\
    	Constrained Branin Currin (heteroskedastic noise) & [-194.9376,  -12.2969]\\
    	Constrained Branin Currin (homoskedastic noise) & [-195.4667,  -12.4984]\\
    	Penicillin & [5.657, -64.1, -340.0] \\
    	GMM ($M=2, \alpha=0.9$, multiplicative noise) & [0.3752, 0.3548] \\
    	GMM ($M=2, \alpha=0.9$, correlated noise) & [0.2727, 0.2583] \\
    	GMM ($M=2, \alpha=0.8$, heteroskedastic noise) & [0.3465, 0.3036]\\
    	GMM ($M=2, \alpha=0.9$, homoskedastic noise, $\sigma=0.05$) & [0.2756, 0.2368]\\
    	GMM ($M=2, \alpha=0.9$, homoskedastic noise, $\sigma=0.1$) & [0.1047, 0.1112]\\
    	GMM ($M=2, \alpha=0.9$, homoskedastic noise, $\sigma=0.2$) & [0.0160, 0.0131]\\
    	GMM ($M=3, \alpha=0.9$, homoskedastic noise, $\sigma=0.05$) & [0.2733, 0.0051, 0.1538]\\
    	GMM ($M=3, \alpha=0.9$, homoskedastic noise, $\sigma=0.05$) & [0.2733, 0.0051, 0.1538]\\
    	GMM ($M=3, \alpha=0.8$, homoskedastic noise, $\sigma=0.05$) & [0.0420, 0.0180, 0.1952]\\
    	GMM ($M=3, \alpha=0.7$, homoskedastic noise, $\sigma=0.05$) & [0.0537, -0.0517, -0.0021]\\
    	GMM ($M=4, \alpha=0.9$, homoskedastic noise, $\sigma=0.05$) & [0.0264, -0.0396, 0.0619, 0.1689]\\
    	GMM ($M=4, \alpha=0.8$, homoskedastic noise, $\sigma=0.05$) & [0.0322, -0.0398, 0.1168, -0.0023]\\
    	
        \bottomrule
    \end{tabular}
    }
\end{table*}

\subsection{Evaluation Details}
\label{appdx:evaluation_details}
The global \mvar{} set is unknown and is approximated by taking the union of the \mvar{} sets of all designs evaluated across all methods and all replications. We take this approach because even using an evolutionary algorithm to optimize \mvar{} is nontrivial, since \mvar{} maps a single design to a set of points and is relatively computationally intensive to evaluate.
To evaluate the performance of a given method, we use $n_{\bm{\xi}}=512$ (except for 4 objective GMM, where we use $n_{\bm{\xi}}=256$) to compute a high-fidelity estimate of the \mvar{} set across the designs selected during optimization by the method. We similarly use the same $n_{\bm{\xi}}=512$ samples to estimate the true \mvar{} set (by considering all designs evaluated across all methods and all replications).

\section{Wall Times}
In Table~\ref{table:wall-clock-time}, we present the time it takes to run a single BO iteration using all algorithms we considered in this paper. We include the runtimes for the four problems from the main text. Wall times for additional problems including several problems with 3 and 4 objectives are provided in Table~\ref{table:additional-runtime} (these problems are described in Appendix~\ref{appdx:additional-problems}). As we discuss in Appendix~\ref{appdx:direct-mvar-opt}, \mvarnehvi{}, when not computationally infeasible, is quite expensive to run, making it an impractical method for most problems. Although \mvarnehvirff{} provides a much cheaper and highly performant approximation, we see that it also runs into computational limitations as the size of the \mvar{} set grows (e.g. for Penicillin with 3 objectives); this is also pronounced in Table~\ref{table:additional-runtime} on the problems with 3 and 4 objectives. For the \mars{} family of methods, we see overall quite reasonable runtimes, with the most expensive one, \marsnei{}, taking on average $41.4$ seconds on the most expensive problem instance we considered. \marsts{} offers a cheaper alternative to \marsnei{}, with its average runtime remaining below $10$ seconds on all experiments. As we show later in the Appendix, the performance of \marsts{} typically trails closely behind \marsnei{}, making it a strong alternative when the algorithm runtime is of the essence.

\FloatBarrier
\begin{table*}[ht]
	\newcommand{\bt}[1]{\bm{\textcolor{OliveGreen}{#1}}}
	\newcommand{\sn}[1]{\bm{\textcolor{RedOrange}{#1}}}
	\newcommand{\td}[1]{\bm{\textcolor{YellowOrange}{#1}}}
	\caption{The wall time (in seconds) per BO iteration. The experiments were timed on a shared cluster using 4 CPU cores, 1 GPU, and 16 GB of RAM. We report the mean and 2 standard errors over 20 trials. An N/A entry denotes that we did not attempt to run a particular experiment (e.g., because the method does not support the problem setting), whereas an OOM entry denotes that we attempted but the experiment did not run due to scalability limitations. The three top-performing algorithms, with respect to the final average \mvar{} \HV{} regret in each experiment are highlighted using $\bt{best}, \sn{second}, \td{third}$, respectively.}
	\label{table:wall-clock-time} 
	\centering
	\small{\tabcolsep=0.11cm
    \begin{tabular}{ccccc}
        \toprule
    	Algorithm              & GMM                       & Constrained BC        & Disc Brake            & Penicillin           \\
        ($d, M, V, \alpha$)    & ($2, 2, 0, 0.9$)          & ($2, 2, 1, 0.7$)      & ($4, 2, 4, 0.95$)     & ($7, 3, 0, 0.8$)      \\
    	\midrule
    	Sobol                  & $0.4 ~(\pm 0.6)$          & $0.9 ~(\pm 1.0)$      & $0.9 ~(\pm 1.0)$      & $2.9 ~(\pm 1.5)$      \\
    	\qNParego{}            & $2.5 ~(\pm 2.3)$          & $5.6 ~(\pm 4.4)$      & $7.9 ~(\pm 15.1)$     & $21.2 ~(\pm 25.6)$    \\
    	\qNEHVI{}              & $2.5 ~(\pm 2.0)$          & $9.8 ~(\pm 8.9)$      & $16.9 ~(\pm 10.3)$    & $23.6 ~(\pm 41.0)$    \\
    	\qNEHVI{}-RFF          & $0.8 ~(\pm 0.7)$          & $2.1 ~(\pm 3.2)$      & $2.7 ~(\pm 5.4)$      & \sn{$5.0 ~(\pm 3.9)$} \\
    	Exp-\qNParego{}        & $3.3 ~(\pm 2.4)$          & $10.7 ~(\pm 11.1)$    & $23.3 ~(\pm 166.3)$   & $121.9 ~(\pm 119.0)$  \\
    	\expnehvirff{}         & $0.9 ~(\pm 1.0)$          & $7.2 ~(\pm 10.8)$     & \bt{$3.3 ~(\pm 7.3)$} & \td{$5.1 ~(\pm 3.3)$} \\
    	\marsnei{}             & $3.9 ~(\pm 3.0)$          & \sn{$8.4 ~(\pm 5.3)$} & $10.3 ~(\pm 45.4)$    & \bt{$41.4 ~(\pm 60.1)$}\\
    	\marsts{}              & \td{$3.3 ~(\pm 3.6)$}     & \td{$3.3 ~(\pm 3.7)$} & $3.1 ~(\pm 6.5)$      & $6.7 ~(\pm 5.8)$    \\
    	\marsucb{}             & $8.2 ~(\pm 11.0)$         & N/A                   & N/A                   & $12.3 ~(\pm 14.1)$   \\
    	\mvarnehvi{}           & \bt{$145.6 ~(\pm 130.7)$} & N/A                   & \sn{$243.2 ~(\pm 154.)$}& N/A                \\
    	\mvarnehvirff{}        & \sn{$1.8 ~(\pm 1.7)$}     & \bt{$10.0 ~(\pm 11.9)$}& \td{$2.9 ~(\pm 3.1)$}& OOM                \\
        \bottomrule
    \end{tabular}
    }
\end{table*}
\FloatBarrier

\section{Additional Experiments}
\label{appdx:additional_experiments}
\subsection{Additional Test Problems}\label{appdx:additional-problems}
In addition to the problems presented in the main text, we studied the performance of the algorithms on the Toy problem used for illustrations in Figures~\ref{fig:1d_toy} and \ref{fig:bijection}, and the 3 and 4 objective variations of the GMM problem. These problems are described in detail in Appendix~\ref{appdx:problem_details}. 
In addition to the acquisition functions presented in the main text, we ran 
\expnehvirff{}, \mvarnehvirff{}, \marsucb{}, and \marsts{}.
The results of these experiments are presented in Figure~\ref{fig:toy-results} for the Toy problem and Figure~\ref{fig:gmm-results_3obj_4obj} for the GMM problems, and the runtimes of the algorithms are reported in Table~\ref{table:additional-runtime}. 
We see that \marsnei{} is overall the best performing method, with \marsts{} typically following closely. \marsucb{} appears to be less reliable, demonstrating significantly worse performance in most experiments. In addition, the \mvarnehvirff{} is missing from all but two of the experiments, which is due to the method running into the scalability limitations discussed in Appendix \ref{appdx:direct-mvar-opt}.

\begin{figure*}[ht]
    \centering
    \includegraphics[width=0.5\linewidth]{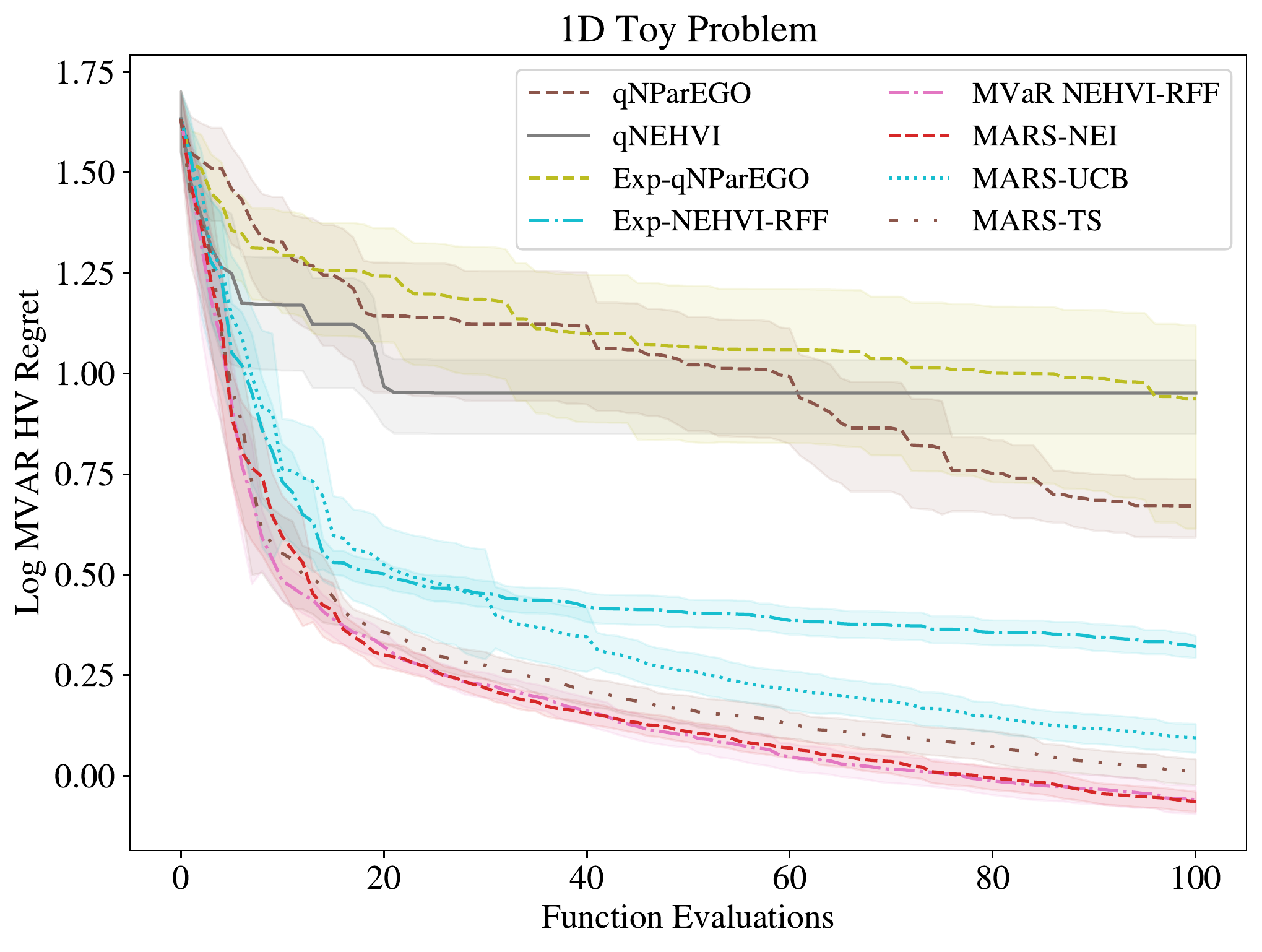}
    \caption{The log \mvar{} hypervolume regret on the Toy problem. We plot means and 2 standard errors across 20 trials.}
    \label{fig:toy-results}
\end{figure*}

\begin{figure*}[ht]
    \centering
    \includegraphics[width=\linewidth]{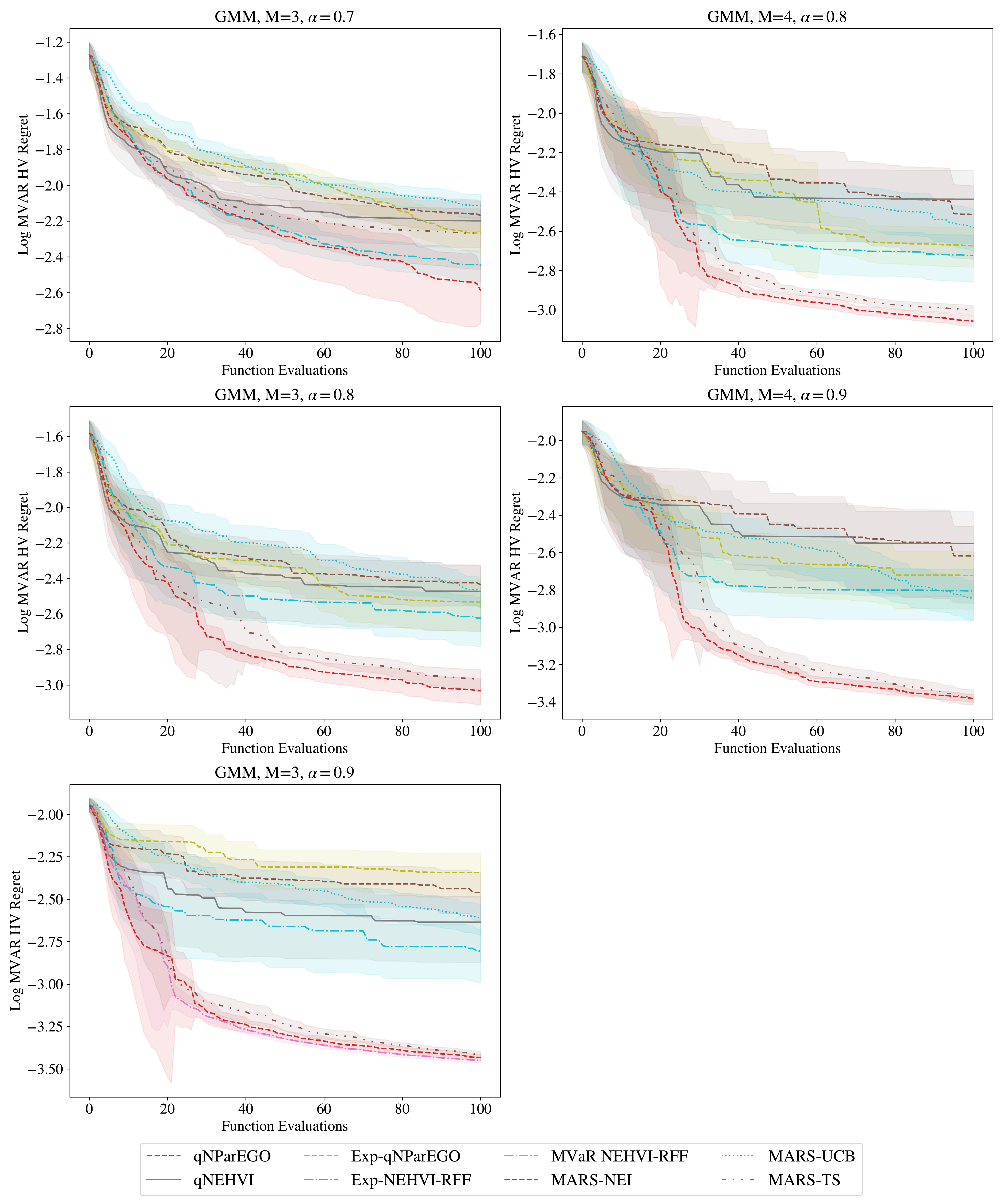}
    \caption{The log \mvar{} hypervolume regret on 3 and 4-Objective GMM problems. We plot means and 2 standard errors across 20 trials.}
    \label{fig:gmm-results_3obj_4obj}
\end{figure*}

\FloatBarrier
\begin{table*}[ht]
	\newcommand{\bt}[1]{\bm{\textcolor{OliveGreen}{#1}}}
	\newcommand{\sn}[1]{\bm{\textcolor{RedOrange}{#1}}}
	\newcommand{\td}[1]{\bm{\textcolor{YellowOrange}{#1}}}
	\caption{The wall time (in seconds) per BO iteration for the additional problems. The experiments were timed on a shared cluster using 4 CPU cores, 1 GPU, and 16 GB of RAM. We report the mean and 2 standard errors over 20 trials. An OOM entry denotes that the experiment did not run due to scalability limitations. The three top-performing algorithms, with respect to the final average \mvar{} \HV{} regret in each experiment are highlighted using $\bt{best}, \sn{second}, \td{third}$, respectively.}
	\label{table:additional-runtime} 
	\centering
	\small{\tabcolsep=0.11cm
    \begin{tabular}{ccccccc}
        \toprule
    	Algorithm          & Toy Problem            & \multicolumn{3}{c}{GMM, M=3}                                            & \multicolumn{2}{c}{GMM, M=4}              \\
        ($d, M, \alpha$)   & ($1, 2, 0.9$)          & ($2, 3, 0.7$)          & ($2, 3, 0.8$)          & ($2, 3, 0.9$)         & ($2, 4, 0.8$)           & ($2, 4, 0.9$)      \\
    	\midrule
    	Sobol              & $0.4 ~(\pm 0.3)$       & $1.9 ~(\pm 2.2)$       & $1.2 ~(\pm 1.4)$       & $0.5 ~(\pm 0.3)$      & $3.3 ~(\pm 3.3)$        & $0.8 ~(\pm 0.2)$ \\
    	\qNParego{}        & $2.2 ~(\pm 2.0)$       & $4.5 ~(\pm 2.6)$       & $3.4 ~(\pm 2.6)$       & $1.4 ~(\pm 1.5)$      & $8.9 ~(\pm 10.2)$       & $4.5 ~(\pm 12.9)$ \\
    	\qNEHVI{}          & $3.0 ~(\pm 2.5)$       & $23.3 ~(\pm 26.4)$     & $20.2 ~(\pm 22.8)$     & $23.8 ~(\pm 43.8)$    & $57.3 ~(\pm 75.4)$      & $48.0 ~(\pm 55.2)$ \\
    	\qNEHVI{}-RFF      & $0.4 ~(\pm 0.3)$       & \td{$3.1 ~(\pm 5.8)$}  & $1.7 ~(\pm 1.0)$       & $0.7 ~(\pm 0.5)$      & $10.0 ~(\pm 91.8)$      & $2.1 ~(\pm 4.0)$ \\
    	Exp-\qNParego{}    & $13.1 ~(\pm 23.5)$     & $8.0 ~(\pm 12.8)$      & $6.2 ~(\pm 6.4)$       & $13.4 ~(\pm 54.2)$    & $10.6 ~(\pm 10.6)$      & $8.1 ~(\pm 18.3)$ \\
    	\expnehvirff{}     & $0.4 ~(\pm 0.3)$       & \sn{$4.0 ~(\pm 10.3)$} & \td{$2.1 ~(\pm 6.3)$}  & $0.8 ~(\pm 0.8)$      & \td{$5.3 ~(\pm 4.0)$}   & $1.9 ~(\pm 1.5)$ \\
    	\marsnei{}         & \bt{$10.4 ~(\pm 21.1)$}& \bt{$11.1 ~(\pm 14.4)$}& \bt{$10.7 ~(\pm 23.0)$}& \sn{$8.8 ~(\pm 13.2)$}& \bt{$27.5 ~(\pm 221.1)$}& \bt{$13.5 ~(\pm 15.9)$} \\
    	\marsts{}          & \td{$0.9 ~(\pm 0.7)$}  & $5.7 ~(\pm 8.9)$       & \sn{$3.8 ~(\pm 2.9)$}  & \td{$1.8 ~(\pm 2.5)$} & \sn{$9.9 ~(\pm 27.1)$}  & \sn{$4.2 ~(\pm 4.2)$} \\
    	\marsucb{}         & $6.7 ~(\pm 7.4)$       & $14.4 ~(\pm 22.5)$     & $10.4 ~(\pm 11.1)$     & $10.3 ~(\pm 12.3)$    & $17.6 ~(\pm 19.1)$      & \td{$12.1 ~(\pm 11.9)$} \\
    	\mvarnehvirff{}    & \sn{$2.6 ~(\pm 2.2)$}  & OOM                    & OOM                    & \bt{$3.0 ~(\pm 2.1)$} & OOM                     & OOM \\
        \bottomrule
    \end{tabular}
    }
\end{table*}
\FloatBarrier

\subsection{Comparison with \qNEHVI{} Based Methods}

\FloatBarrier
\begin{table*}[ht]
	\newcommand{\bt}[1]{\bm{\textcolor{OliveGreen}{#1}}}
	\newcommand{\sn}[1]{\bm{\textcolor{RedOrange}{#1}}}
	\newcommand{\td}[1]{\bm{\textcolor{YellowOrange}{#1}}}
	\caption{The final \mvar{} HV regret obtained using each method. We report the mean and 2 standard errors over 20 trials. An N/A entry denotes that we did not attempt to run a particular experiment (e.g., because the method does not support the problem setting), whereas an OOM entry denotes that we attempted but the experiment did not run due to scalability limitations. The three top-performing algorithms, with respect to the final average \mvar{} \HV{} regret in each experiment are highlighted using $\bt{best}, \sn{second}, \td{third}$, respectively.}
	\label{table:final_regret} 
	\centering
	\small{\tabcolsep=0.11cm
    \begin{tabular}{ccccccc}
        \toprule
    	Algorithm              & GMM                    & Constrained BC         & Disc Brake             & Penicillin             & Toy Problem            & GMM, M=4               \\
        ($d, M, V, \alpha$)    & ($2, 2, 0, 0.9$)       & ($2, 2, 1, 0.7$)       & ($4, 2, 4, 0.95$)      & ($7, 3, 0, 0.8$)       & ($1, 2, 0, 0.9$)       & ($2, 4, 0, 0.8$)       \\
        Scale                  & $1 \times 10^{-4}$     & $1 \times 10^{1}$      & $1 \times 10^{0}$      & $1 \times 10^{4}$      & $1 \times 10^{0}$      & $1 \times 10^{-3}$     \\
    	\midrule
    	Sobol                  & $65.37 ~(\pm 10.23)$   & $3.12 ~(\pm 0.36)$     & $19.39 ~(\pm 1.35)$    & $1.68 ~(\pm 0.05)$     & $2.38 ~(\pm 0.27)$     & $7.39 ~(\pm 1.35)$     \\
    	\qNParego{}            & $21.54 ~(\pm 7.68)$    & $3.28 ~(\pm 0.41)$     & $19.10 ~(\pm 1.83)$    & $1.73 ~(\pm 0.09)$     & $4.69 ~(\pm 0.78)$     & $3.04 ~(\pm 1.24)$     \\
    	\qNEHVI{}              & $31.21 ~(\pm 14.34)$   & $5.16 ~(\pm 0.76)$     & $12.33 ~(\pm 1.77)$    & $1.65 ~(\pm 0.12)$     & $8.95 ~(\pm 1.87)$     & $3.66 ~(\pm 1.47)$     \\
    	\qNEHVI{}-RFF          & $30.54 ~(\pm 12.13)$   & $5.60 ~(\pm 0.81)$     & $17.11 ~(\pm 1.97)$    & \sn{$1.38 ~(\pm 0.09)$}& $5.97 ~(\pm 0.77)$     & $2.18 ~(\pm 0.62)$     \\
    	Exp-\qNParego{}        & $10.29 ~(\pm 2.92)$    & $2.73 ~(\pm 0.77)$     & $3.43 ~(\pm 0.57)$     & $1.52 ~(\pm 0.10)$     & $8.65 ~(\pm 4.53)$     & $2.12 ~(\pm 0.47)$     \\
    	\expnehvirff{}         & $13.83 ~(\pm 5.95)$    & $1.23 ~(\pm 0.05)$     & \bt{$1.02 ~(\pm 0.05)$}& \td{$1.46 ~(\pm 0.06)$}& $2.09 ~(\pm 0.13)$     & \td{$1.90 ~(\pm 0.50)$}\\
    	\marsnei{}             & $1.45 ~(\pm 0.12)$     & \sn{$0.70 ~(\pm 0.04)$}& $2.96 ~(\pm 0.13)$     & \bt{$1.06 ~(\pm 0.09)$}& \bt{$0.86 ~(\pm 0.05)$}& \bt{$0.88 ~(\pm 0.05)$}\\
    	\marsts{}              & \td{$0.92 ~(\pm 0.07)$}& \td{$0.85 ~(\pm 0.05)$}& $3.20 ~(\pm 0.21)$     & $1.49 ~(\pm 0.05)$     & \td{$1.02 ~(\pm 0.08)$}& \sn{$1.00 ~(\pm 0.06)$}\\
    	\marsucb{}             & $26.69 ~(\pm 3.28)$    & N/A                    & N/A                    & $1.77 ~(\pm 0.01)$     & $1.24 ~(\pm 0.10)$     & $2.61 ~(\pm 0.64)$     \\
    	\mvarnehvi{}           & \bt{$0.57 ~(\pm 0.01)$}& N/A                    & \sn{$1.02 ~(\pm 0.04)$}& N/A                    & $1.76 ~(\pm 0.15)$     & N/A                    \\
    	\mvarnehvirff{}        & \sn{$0.64 ~(\pm 0.03)$}& \bt{$0.57 ~(\pm 0.02)$}& \td{$1.92 ~(\pm 0.06)$}& OOM                    & \sn{$0.87 ~(\pm 0.07)$}& OOM                    \\
        \bottomrule
    \end{tabular}
    }
\end{table*}
\FloatBarrier

As discussed in Section \ref{sec:opt_mvar} and detailed in Appendix \ref{appdx:direct-mvar-opt}, \mvar{} can be optimized using \qNEHVI{} based methods. However, this comes with serious computational challenges, some of which are eased if we use \qNEHVI{} with RFF draws. 
In this section, we present results comparing \marsnei{} with
\expnehvirff{} and \mvarnehvirff{} using the test problems from the main text. \mvarnehvi{} was also ran on two of the problems to provide a point of reference for its performance.
In addition, we present both the standard \qNEHVI{}, which uses GPs, as well as its RFF counterpart as a reference point on 
how the use of RFFs affects the performance of the methods on a given problem.

\begin{figure*}[ht]
    \centering
    \includegraphics[width=\linewidth]{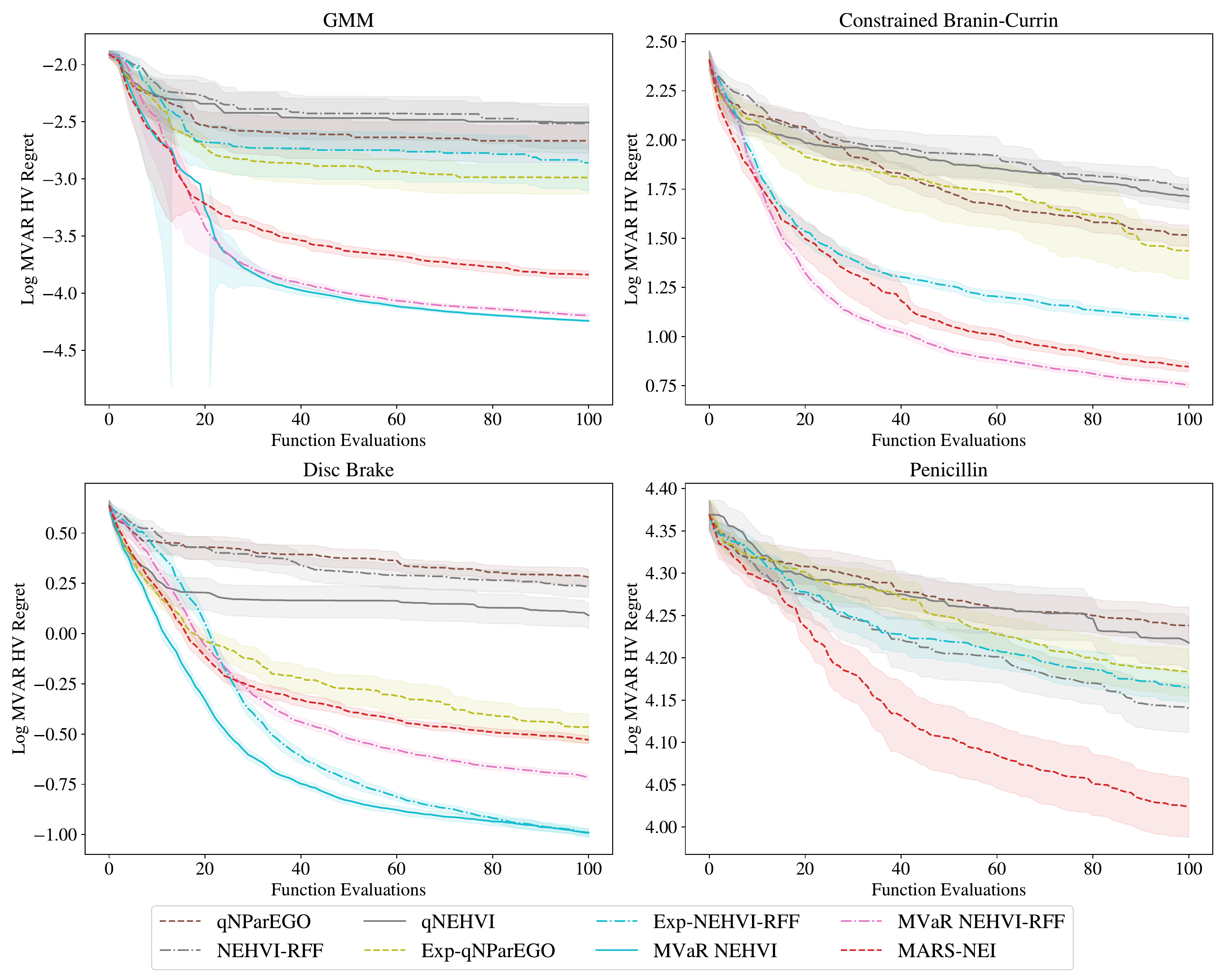}
    \caption{A comparison of the methods from evaluated in the main text with additional \qNEHVI{}-based methods: \textsc{NEHVI-RFF}, \expnehvirff{}, \mvarnehvi{}, and \mvarnehvirff{}.}
    \label{fig:nehvi-results}
\end{figure*}

Figure \ref{fig:nehvi-results} and Table~\ref{table:final_regret} show that although there are instances where \qNEHVI{} methods outperform \marsnei{}, \marsnei{} remains competitive throughout. We see that \qNEHVI{} and its RFF counterpart are competitive, without a clear winner across problems. The expectation \qNEHVI{} outperforms the expectation \qNParego{} in all but one problem, highlighting the benefit of using a method that aims to directly maximize the hypervolume in an MO setting. Lastly, it is worth highlighting that the \mvarnehvirff{} is not reported for the Penicillin problem, which is due to the scalability limitations of IEP (discussed in Appendix~\ref{appdx:direct-mvar-opt}), preventing the hypervolume improvement computations from running with the available GPU memory. 
See Appendix~\ref{appdx:additional-problems} for additional experiments comparing \qNEHVI{} based methods.

\subsection{Comparison of Methods Optimizing \mvar{}}
In Figure~\ref{fig:mvar-results} and Table~\ref{table:final_regret}, we present results from the test problems from the main text showing the performance of all acquisition functions that we proposed for optimizing \mvar{}. Results for additional problems comparing are presented in Appendix~\ref{appdx:additional-problems}. We observe that \mvarnehvirff{} is a reasonably cheap method (see Table~\ref{table:wall-clock-time} for runtimes) that performs quite well on smaller problem instances. However, in larger problem instances where the size of the \mvar{} set is large ($>10$) it starts running into scalability limitations and no-longer works. Among the class of \mars{} methods, we find that \marsnei{} consistently performs quite well, with \marsts{} being a close second and a slightly cheaper alternative. On the other hand, \marsucb{} proves to be not as reliable, performing worse than non-robust methods in some problems. We attribute this to its fundamental reliance on the parameter $\zeta^{(i)}_{n+1}$, which would have to be tuned on a problem by problem basis to optimize its performance.  In light of all the results, we recommend \marsnei{} as a broadly applicable and high performing method for optimizing \mvar{}, and recommend \mvarnehvirff{} as an alternative when the size of the \mvar{} set is small.

\begin{figure*}[ht]
    \centering
    \includegraphics[width=\linewidth]{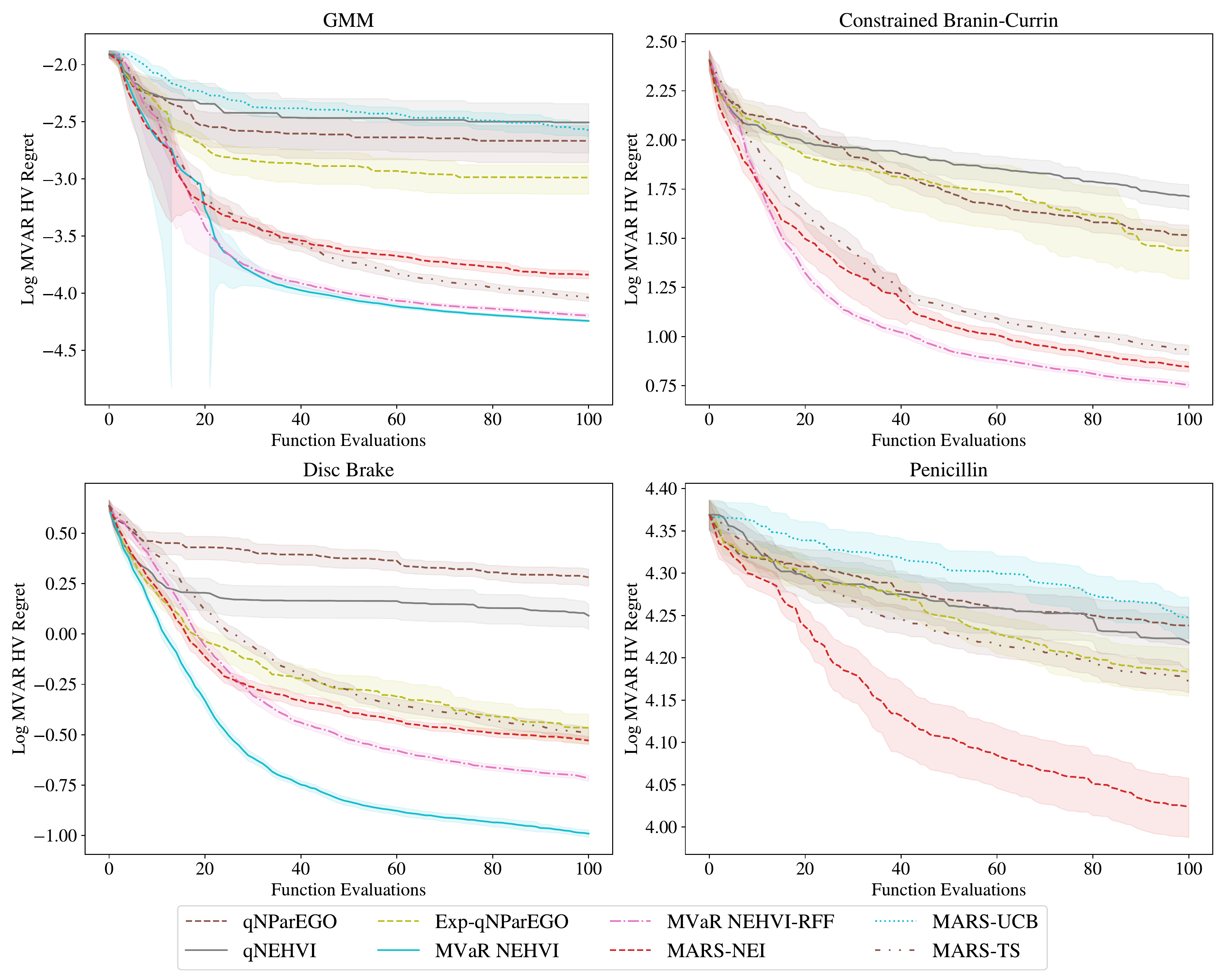}
    \caption{A comparison of the methods from evaluated in the main text with additional methods for optimizing \mvar{}: \marsts{}, \marsucb{}, \mvarnehvirff{}, and \mvarnehvi{}.}
    \label{fig:mvar-results}
\end{figure*}

\subsection{Parallel Evaluations}
\label{appdx:batch_results}
In Table \ref{table:batch-size-results}, we present results demonstrating the effect of varying batch size on the optimization performance of \marsnei{} and \mvarnehvirff{}. As expected, the results show that the performance of both algorithms degrade as the batch size increases. Interestingly, the degradation is rather minimal for \marsnei{}, while the effects of increasing batch size seem to be rather significant for \mvarnehvirff{}. We see that \mvarnehvirff{} underperforms all versions of \marsnei{} even with a batch size of $2$. In addition, there are fewer results presented for \mvarnehvirff{}, which is due to the hypervolume improvement computations using \iep{} running into scalability limits (recall that \iep{} scales exponentially in the size of the joint \mvar{} set of the current batch of candidates). These results make a strong case for using \marsnei{} whenever one wished to evaluate candidates in parallel.


\begin{table*}[ht]
	\caption{Effect of the batch size on optimization performance. We report the final \mvar{} HV regret and 2 standard errors from 20 trials. OOM denotes that the method ran into scalability issues and did not run.}
	\label{table:batch-size-results} 
	\centering
    \begin{tabular}{ccc}
        \toprule
    	\multirow{2}{*}{Algorithm}      & 1D Toy Problem             & Constrained Branin-Currin     \\
    	                                & ($d=1, M=2$)               & ($d=2, M=2, V=1$)             \\
    	\midrule
    	\marsnei{}, $q=1$               & $0.86 ~(\pm 0.05)$         & $5.43 ~(\pm 0.27)$ \\
    	\marsnei{}, $q=2$               & $0.92 ~(\pm 0.04)$         & $5.20 ~(\pm 0.21)$ \\
    	\marsnei{}, $q=4$               & $0.92 ~(\pm 0.07)$         & $5.58 ~(\pm 0.38)$ \\
    	\marsnei{}, $q=8$               & $0.93 ~(\pm 0.08)$         & $5.59 ~(\pm 0.41)$ \\
    	\mvar{} \qNEHVI{} RFF, $q=1$    & $0.87 ~(\pm 0.07)$         & $4.36 ~(\pm 0.13)$ \\
    	\mvar{} \qNEHVI{} RFF, $q=2$    & $1.03 ~(\pm 0.05)$         & $6.16 ~(\pm 0.36)$ \\
    	\mvar{} \qNEHVI{} RFF, $q=4$    & $1.21 ~(\pm 0.07)$         & OOM \\
    	\mvar{} \qNEHVI{} RFF, $q=8$    & OOM                        & OOM \\
        \bottomrule
    \end{tabular}
\end{table*}

\subsection{Effect of Noise Level} \label{appdx:gmm_noise_level}
The location of robust designs on a problem depends on many factors, including the magnitude of the input noise. In the edge case where there is no input noise present, the robust designs will be the same as the nominal Pareto optimal designs. As the magnitude of the input noise increases, the robust designs may start to deviate from the nominally optimal designs, with the exact behavior typically being unpredictable and heavily dependent on the problem and noise structure. In Figure \ref{fig:noise-level-results}, we present results on the GMM problem from the main text under various noise models, demonstrating how the performance of the algorithms change in response to changes in the noise model. The list of noise models considered in this study are as follows:
\begin{itemize}[noitemsep]
    \item Homoscedastic normal noise, std = 0.05: $P(\bm \xi) = \mathcal{N}(\mu = \bm 0, \Sigma = 0.05 I_2)$.
    \item Homoscedastic normal noise, std = 0.10: $P(\bm \xi) = \mathcal{N}(\mu = \bm 0, \Sigma = 0.10 I_2)$.
    \item Homoscedastic normal noise, std = 0.20: $P(\bm \xi) = \mathcal{N}(\mu = \bm 0, \Sigma = 0.20 I_2)$.
    \item Heteroscedastic normal noise, std = 0.2X: $P(\bm \xi; \bm x) = \mathcal{N}(\mu = \bm 0, \Sigma = 0.2 S)$ with $S = [x_1, 0; 0, x_2]$ is the $2 \times 2$ matrix with the given entries.
    \item Correlated normal noise: $P(\bm \xi) = \mathcal{N}(\mu = \bm 0, \Sigma = 0.001 S)$ with $S = [2.5, -2; -2, 2.5]$.
    \item Multiplicative noise model from the main text: $\xxi := \bm x \bm \xi'$, where $\bm \xi' \sim \mathcal{N}(\mu=\bm 1, \Sigma = 0.07 I_2)$.
\end{itemize}

We see that \marsnei{} consistently outperforms the alternatives, except for the under homoskedastic noise with a large standard deviation of $0.2$. In this setting, the \textsc{Exp-}\qNParego{} is slightly ahead, which is not too surprising since the expectation and \mvar{} optimal designs happen to be in the same part of the solution space under this noise model.

\begin{figure*}[ht]
    \centering
    \includegraphics[width=\linewidth]{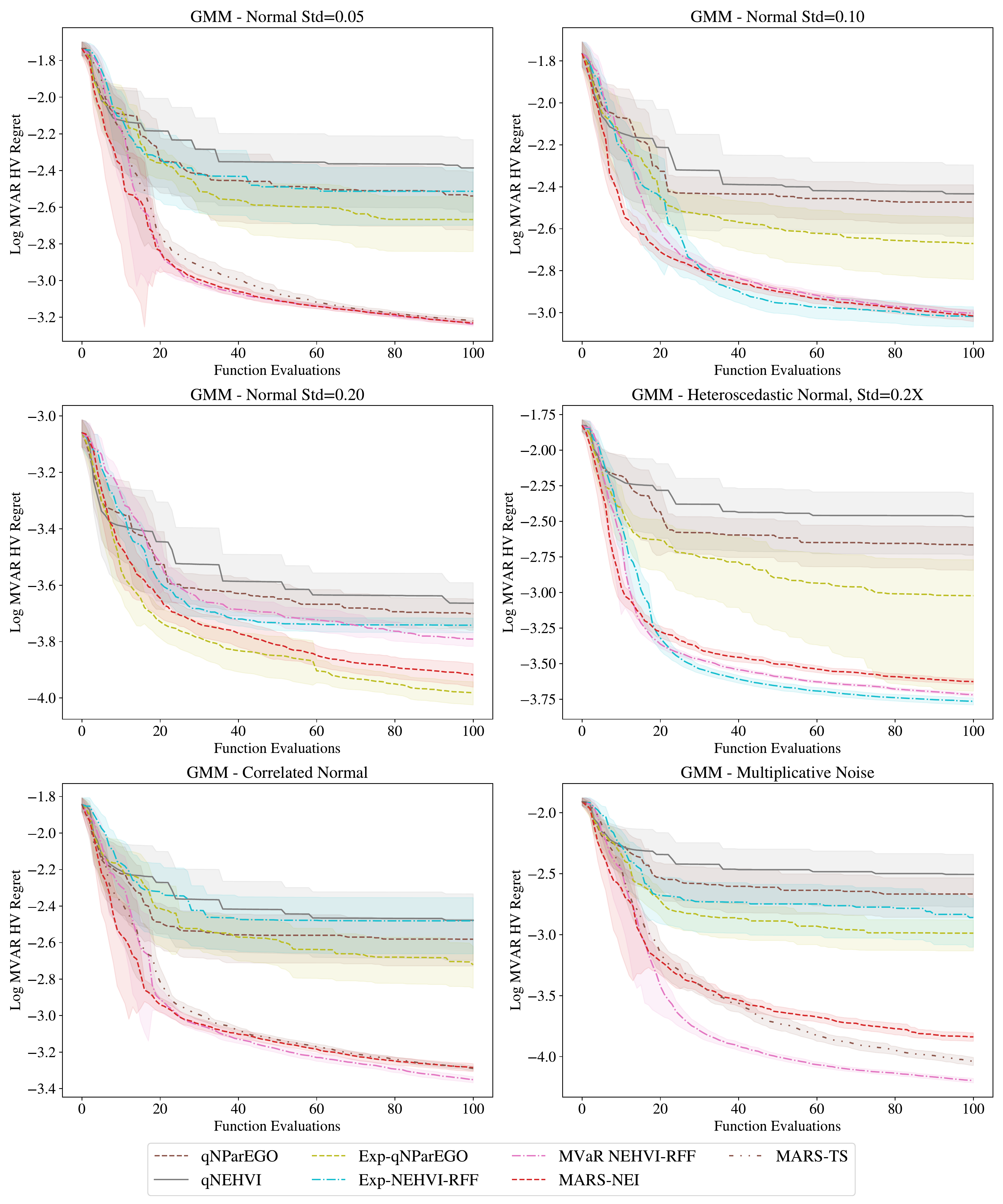}
    \caption{The effect of different noise models on the GMM problem.}
    \label{fig:noise-level-results}
\end{figure*}

In addition to the GMM problem, the Constrained Branin Currin problem was also ran under multiple noise models. Along with the heteroscedastic noise model used in the main text, we also studied it using a simple homoscedastic noise model: $P(\bm \xi) = \mathcal{N}(\mu = \bm 0, \Sigma = 0.05 I_2)$. The plots for these are shown in Figure~\ref{fig:bc-noise-level-results}, demonstrating that \mars{} performs consistently under both noise models.

\begin{figure*}[ht]
    \centering
    \includegraphics[width=\linewidth]{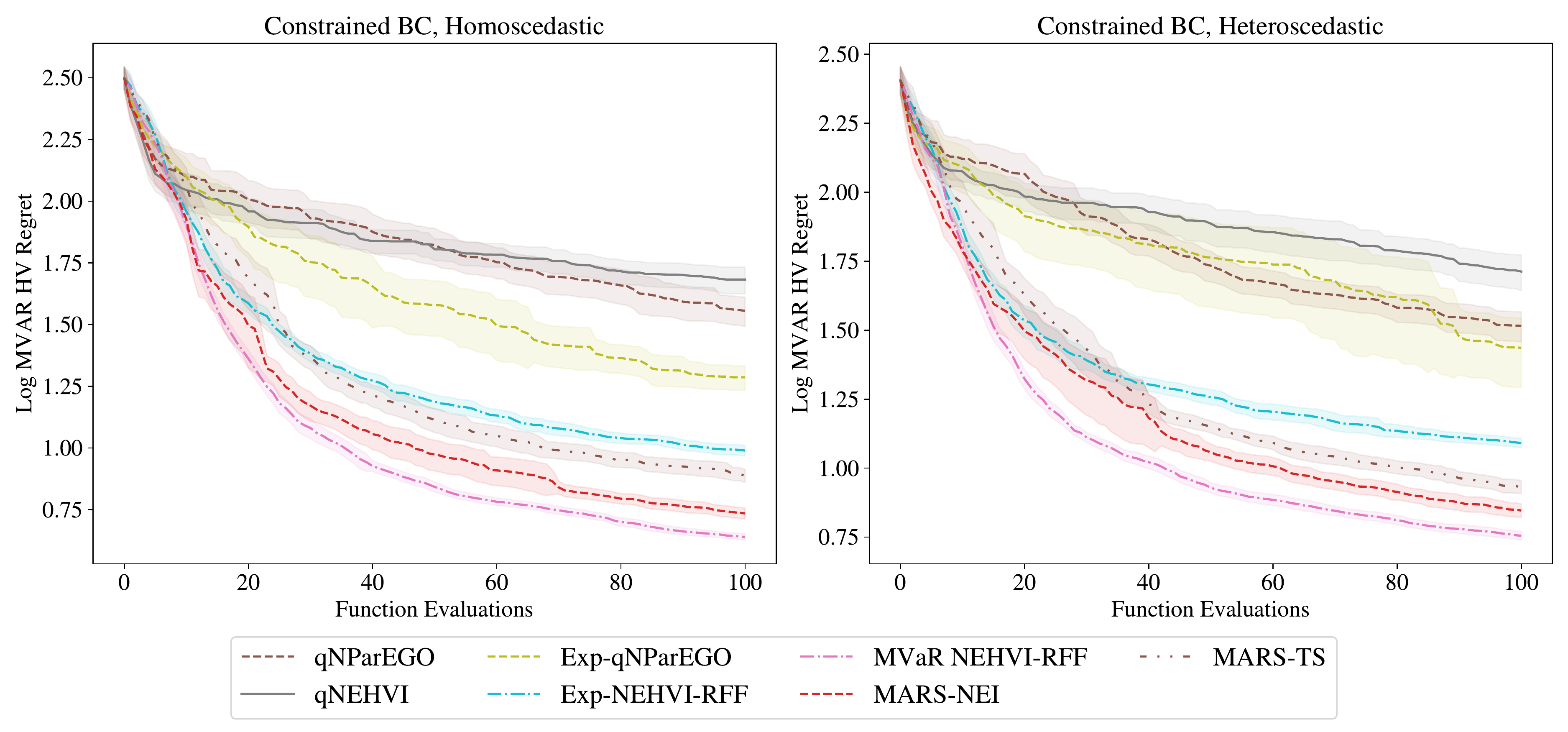}
    \caption{The effect of different noise models on the Constrained Branin Currin problem.}
    \label{fig:bc-noise-level-results}
\end{figure*}

\subsection{Effect of $n_{\bm{\xi}}$ on Optimization Performance}
In a final side study, we analyze the effect of varying $n_{\bm\xi}$ on the optimization performance of the acquisition functions optimizing expectation and \mvar{}. We attempted to run the Constrained Branin Currin and Disc Brake experiments with $n_{\bm\xi} \in \{8, 16, 32, 64, 96, 128\}$ and included the results of the algorithms that successfully completed without running into scalability issues, such as getting an out-of-memory error.

\begin{figure}[ht]
\def\nxiplotwidth{0.4\linewidth}
    \centering
    \includegraphics[width=\nxiplotwidth]{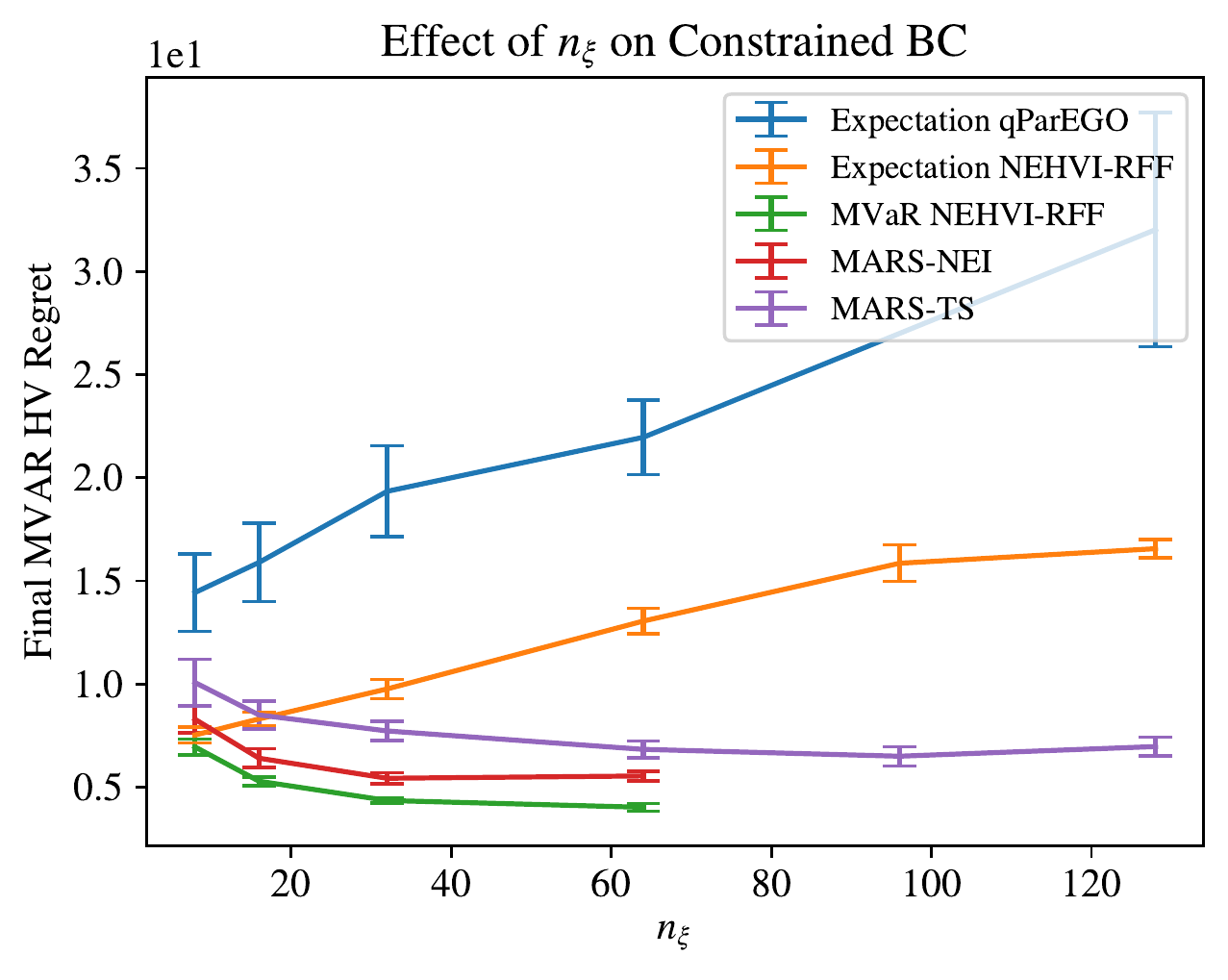}
    \includegraphics[width=\nxiplotwidth]{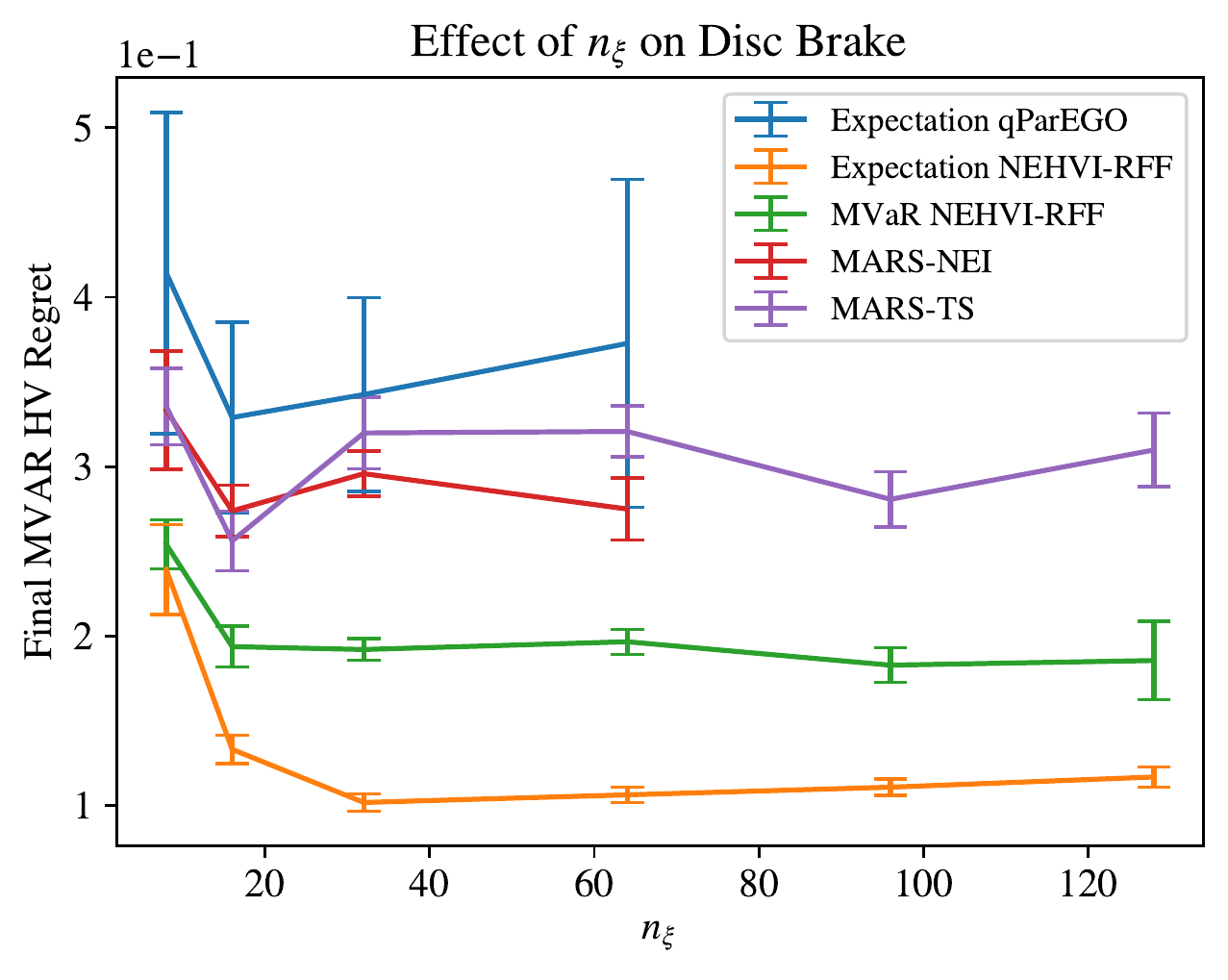}
    \caption{Final \mvar{} hypervolume regret obtained using different $n_{\bm \xi}$ on Constrained Branin-Currin and Disc Brake.}
    \label{fig:n_xi_plots}
\end{figure}

The results are shown in Figure \ref{fig:n_xi_plots}.
We see that the methods using GPs (\textsc{Exp-}\qNParego{} and \marsnei{}), do not scale beyond $n_{\bm\xi} = 64$ due to the cubic  complexity (with respect to the number of points and $n_{\bm\xi}$) of posterior sampling (see Appendix~\ref{appdx:complexity}, the remaining methods avoid this issue since the RFF draws are deterministic functions). 
Overall, the results for \mvar{} methods show that too small of an $n_{\bm{\xi}}$ leads to poor performance, while increasing it much beyond our default value of $n_{\bm{\xi}}=32$ does not yield any a significant benefit, at least in these problems. On Constrained Branin-Currin, we observe that the performance of the expectation methods degrade as $n_{\bm\xi}$ increases, which we attribute to these methods becoming more proficient at differentiating the expectation and \mvar{} optimal regions (which are not co-located), thus focusing their sampling away from the \mvar{} optimal region.

\section{Efficient Methods for Computing
\mvar{}}\label{appdx:efficient_mvar}
Before going into the discussion, we note that the this section presents the \mvar{} computation for a random variable to be minimized. This simplifies the discussion by enabling the use of common terms such as CDF and quantile, since the \mvar{}, as originally defined in \citet{Prekopa2012MVaR}, corresponds to the $\alpha$ quantile of a random variable to be minimized. In the maximization setting studied in this paper, the \mvar{}, as defined in Definition~\ref{def:mvar}, can be computed by first computing the \mvar{} of the negative of the random variable, as discussed here, then negating the result.

The existing algorithms for computing the \mvar{} (e.g., those presented by \citet{Prekopa2012MVaR}) presume the availability of a cheap to evaluate CDF of the random variable of interest. In the general setting we consider, with $\bm f$ being an arbitrary function, such as a sample path of the GP, the random variable $\bm f (\xxi)$ (induced by $\bm \xi \sim P(\bm \xi)$) does not admit a known CDF. 
Thus, to compute a QMC estimate of \mvar{}, we first need to compute the empirical CDF corresponding to the QMC samples of $\bm f (\xxi)$.

Computing the empirical CDF of a random variable is a conceptually simple operation. All we need to do is to count the number of samples that are dominated by a given point and divide that by the total number of samples. Keeping all other objectives constant, the empirical CDF can be seen as a step function over the domain of the given objective that changes its value only at the points that correspond to one of the sample values of that objective. Since there are $n_{\bm{\xi}}$ samples, ignoring the possibility that some samples may have equal value for some objectives, this defines an $M$-dimensional grid with $n_{\bm{\xi}}$ points on each dimension, on which the empirical CDF can change its value. Thus, to fully compute the empirical CDF, we need to compute the number of samples that dominate a grid of $n_{\bm{\xi}}^M$ points, at a cost of $\mathcal{O}(M)$ comparison per point, leading to a total $\mathcal{O}(n_{\bm{\xi}}^M M)$ cost for computing the empirical CDF.

Once the empirical CDF is computed, the \mvar{} set can easily be computed by taking the Pareto set of the points on the grid with a CDF greater than or equal to $\alpha$. This part of the computation, fortunately, has a lower complexity than the CDF computations.

A careful reader might have noticed that the \mvar{} (in the minimization setting) is bounded from below by the independent \var{} of each objective and bounded from above by the maximum value observed for that given objective. We can leverage this fact to lower the cost of computing \mvar{} significantly. Instead of considering the full grid of $n_{\bm{\xi}}^M$ points, we can only compute the empirical CDF for the grid formed by the objective values that exceed the independent \var{} of each objective, of which there are $(1-\alpha) n_{\bm{\xi}}$, reducing the complexity of \mvar{} computations to $\mathcal{O}((1-\alpha)^M n_{\bm{\xi}}^M M)$. 

Within our code base, we provide two implementations for computing \mvar{}. One implementation is geared towards batched calculations with small to moderate $n_{\bm{\xi}}$, and the other is geared towards less memory intensive calculations with large $n_{\bm{\xi}}$. Both utilize efficient vectorized computations to exploit modern computing hardware. However, even with these highly optimized implementations, computing \mvar{} can easily become a bottleneck in a BO method, since \mvar{} has to be computed many times during acquisition optimization.
Thus, the methods for direct optimization of \mvar{} that we present can still be prohibitively expensive unless the objective evaluations are significantly expensive.

\section{Multivariate Extensions of \cvar{}} \label{appdx:multivariate-cvar}
\cvar{} is another popular risk measure that is commonly used with univariate random variables \citep{ROCKAFELLAR2002CVaR}. Similar to \var{}, it has also been extended to the multi-variate case by various authors (e.g., \citet{Cousin2014MCVaR} and \citet{MERAKLI2018300}). 
However, these extensions in general do not admit a natural interpretation, whereas \mvar{} provides interpretable objective specifications that the objectives (under input noise) for a given design will meet with high-probability. 
\end{document}